\documentclass[twoside,11pt]{article}

\usepackage{jmlr2e}

\usepackage{silence}
\usepackage[utf8]{inputenc} 
\usepackage[T1]{fontenc}    
\usepackage{url}            
\usepackage{booktabs}       
\usepackage{amsfonts}       
\usepackage{nicefrac}       
\usepackage{microtype}      
\usepackage{xcolor}         

\usepackage{amsmath}
\usepackage{amssymb}
\usepackage{mathtools}
\usepackage[]{amsbsy}
\usepackage{commath}
\usepackage{thm-restate} 
\usepackage{xspace} 
\usepackage{algorithm}
\usepackage[noend]{algorithmic}

\usepackage{caption}
\usepackage{subcaption}
\usepackage{placeins}
\usepackage{bm}
\usepackage{comment}

\usepackage{ctable}

\usepackage{adjustbox}

\usepackage{hyperref}
\usepackage[capitalize,noabbrev]{cleveref}
\crefname{equation}{}{}

\renewcommand{\paragraph}[1]{\textbf{#1}\hspace{1em}}

\newcommand{\ourmethod}{\textsc{ISE-BO}\xspace}
\newcommand{\ourmethodlong}{Information-Theoretic Safe Exploration and Optimization\xspace}
\newcommand{\ourmethodexp}{\textsc{ISE}\xspace}
\newcommand{\ourmethodlongexp}{Information-Theoretic Safe Exploration\xspace}
\newcommand{\mes}{\textsc{MES}\xspace}
\newcommand{\safeopt}{\textsc{SafeOpt}\xspace}
\newcommand{\stageopt}{\textsc{StageOpt}\xspace}
\newcommand{\linebo}{\textsc{LineBO}\xspace}

\DeclareMathOperator*{\argmax}{arg\,max}
\newcommand{\cpsi}[0]{\mathbb{I}_{s(\cdot) \geq 0}}
\DeclareMathOperator\erf{erf}

\let\originalleft\left
\let\originalright\right
\renewcommand{\left}{\mathopen{}\mathclose\bgroup\originalleft}
\renewcommand{\right}{\aftergroup\egroup\originalright}

\usepackage{lastpage}
\jmlrheading{23}{2023}{1-\pageref{LastPage}}{12/23; Revised XXX}{XXX}{XXX}{Alessandro G. Bottero, Carlos E. Luis, Julia Vinogradska, Felix Berkenkamp, and Jan Peters.}

\ShortHeadings{Information-Theoretic Safe BO}{Bottero, Luis, Vinogradska, Berkenkamp, and Peters}
\firstpageno{1}

\begin{document}

\title{Information-Theoretic Safe Bayesian Optimization}

\author{\name Alessandro G. Bottero \email alessandrogiacomo.bottero@bosch.com \\
    \addr Bosch Corporate Research, TU Darmstadt\\
    Robert-Bosch-Campus 1, 71272 Renningen (Germany)
    \AND
    \name Carlos E. Luis \email carlosenrique.luisgoncalves@bosch.com \\
    \addr Bosch Corporate Research, TU Darmstadt
    \AND
    \name Julia Vinogradska \email julia.vinogradska@bosch.com \\
    \addr Bosch Corporate Research
    \AND
    \name Felix Berkenkamp \email felix.berkenkamp@bosch.com \\
    \addr Bosch Corporate Research
    \AND
    \name Jan Peters \email jan.peters@tu-darmstadt.de \\
    \addr TU Darmstadt, German Research Center for AI (DFKI), Hessian.AI
}

\maketitle

\begin{abstract}
We consider a sequential decision making task, where the goal is to optimize an unknown function without evaluating parameters that violate an \textit{a~priori} unknown (safety) constraint. A common approach is to place a Gaussian process prior on the unknown functions and allow evaluations only in regions that are safe with high probability. Most current methods rely on a discretization of the domain and cannot be directly extended to the continuous case. Moreover, the way in which they exploit regularity assumptions about the constraint introduces an additional critical hyperparameter. In this paper, we propose an information-theoretic safe exploration criterion that directly exploits the GP posterior to identify the most informative safe parameters to evaluate. The combination of this exploration criterion with a well known Bayesian optimization acquisition function yields a novel safe Bayesian optimization selection criterion. Our approach is naturally applicable to continuous domains and does not require additional explicit hyperparameters. We theoretically analyze the method and show that we do not violate the safety constraint with high probability and that we learn about the value of the safe optimum up to arbitrary precision. Empirical evaluations demonstrate improved data-efficiency and scalability.
\end{abstract}

\begin{keywords}
  Bayesian Optimization, Safe Exploration, Safe Bayesian Optimization, Information-Theoretic Acquisition, Gaussian Processes
\end{keywords}

\section{Introduction}\label{sec:ise_introduction}

In sequential decision making problems, we iteratively select parameters in order to optimize a given performance criterion. However, real-world applications such as robotics \citep{berkenkamp_bayesian_2020}, mechanical systems \citep{combustion_engine_bo} or medicine \citep{sui_safe_2015} are often subject to additional safety constraints that we cannot violate during the exploration process \citep{dulac-arnold_challenges_nodate}. Since it is \textit{a priori} unknown which parameters lead to constraint violations, we need to actively and carefully learn about the constraints without violating them. That is, in order to find the safe parameters that maximize the given performance criterion, we also need to learn about the safety of parameters by only evaluating parameters that are known to be safe. In this context, the well known exploration-exploitation dilemma becomes even more challenging, since the exploration component of any algorithm must also promote the exploration of those areas that provide information about the safety of currently uncertain parameters, even though they might be unlikely to contain the optimum. 

Existing methods by \citet{Schreiter_safe_exp_2015,sui_safe_2015,berkenkamp_bayesian_2020} tackle this safe exploration problem by placing a Gaussian process (GP) prior over both the constraint and only evaluate parameters that do not violate the constraint with high probability. To learn about the safety of parameters, they evaluate the parameter with the largest posterior standard deviation. This process is made more efficient by \safeopt, which restricts its safe set expansion exploration component to parameters that are close to the boundary of the current set of safe parameters \citep{sui_safe_2015} at the cost of an additional tuning hyperparameter (Lipschitz constant). However, uncertainty about the constraint is only a proxy objective that only indirectly learns about the safety of parameters. Consequently, data-efficiency could be improved with an exploration criterion that directly maximizes the information gained about the safety of parameters.

\paragraph{Our contribution}
In this paper, we propose \ourmethodlongexp (\ourmethodexp), a safe exploration criterion that \emph{directly} exploits the information gain about the safety of parameters in order to expand the region of the parameter space that we can classify as safe with high confidence. We pair this safe exploration criterion with the well known Max-Value Entropy Search (\mes) acquisition function that aims to find the optimum by reducing the entropy of the distribution of its value. The resulting algorithm, which we call \ourmethodlong (\ourmethod), directly optimizes for information gain, both about the safety of parameters and about the value of the optimum. That way, it is more data-efficient than existing approaches without manually restricting evaluated parameters to be in specifically designed areas (like the boundary of the safe set). This property is particularly evident in scenarios where the posterior variance alone is not enough to identify good evaluation candidates, as in the case of heteroskedastic observation noise. The proposed selection criterion also means that we do not require additional modeling assumptions beyond the GP posterior and that \ourmethod is directly applicable to continuous domains. 
\citet{bottero_info-theoretic_2022} present a partial theoretical analysis of the safe exploration \ourmethodexp algorithm, proving that it possesses some desired exploration properties. In this paper, we extend their theoretical analysis and show that it also satisfies some natural notion of convergence, in the sense that it leads to classifying as safe the largest region of the domain that we can hope to learn about in a safe manner. Subsequently, we use this result to show that \ourmethod learns about the safe optimum to arbitrary precision.

\subsection{Related work} 

Information-based selection criteria with Gaussian processes models are successfully used in the context of unconstrained Bayesian optimization (BO, \citet{shahriari_taking_2016,DBLP:journals/corr/abs-1204-5721}), where the goal is to find the parameters that maximize an \textit{a priori} unknown function.  \citet{HennigS2012,hernandez-lobato_predictive_2014,wang_max-value_2017} select parameters that provide the most information about the optimal parameters, while \citet{frohlich_noisy-input_2020} consider the information under noisy parameters. \citet{hvarfner_joint_entropy_2022}, on the other hand, consider the entropy over the joint optimal probability density over both input and output space. We draw inspiration from these methods and define an information-based criterion w.r.t.\@ the safety of parameters to guide safe exploration.

In the presence of constraints that the final solution needs to satisfy, but which we can violate during exploration, \citet{constrained_bo} propose to combine typical BO acquisition functions with the probability of satisfying the constraint. Instead, \citet{lse_krause} propose an uncertainty-based criterion that learns about the feasible region of parameters, while \citet{constrained_mes_perrone_2019} modify the Max-value Entropy Search selection criterion \citep{wang_max-value_2017} to include information about the constraint as well in the acquisition function.
When we are not allowed to ever evaluate unsafe parameters, safe exploration is a necessary sub-routine of BO algorithms to learn about the safety of parameters. To safely explore, \citet{Schreiter_safe_exp_2015} globally learn about the constraint by evaluating the most uncertain parameters. \safeopt by \citet{sui_safe_2015} extends this to joint exploration and optimization and makes it more efficient by explicitly restricting safe exploration to the boundary of the safe set. \citet{sui2018stagewise} proposes \stageopt, which additionally separates the exploration and optimization phases. Both of these algorithms assume access to a Lipschitz constant to define parameters close to the boundary of the safe set, which is a difficult tuning parameter in practice. These methods have been extended to multiple constraints by \citet{berkenkamp_bayesian_2020}, while \citet{kirschner_adaptive_2019} scale them to higher dimensions with \linebo, which explores in low-dimensional sub-spaces. To improve computational costs, \citet{duivenvoorden_constrained_2017} suggest a continuous approximation to \safeopt without providing exploration guarantees. 
\citet{kimberly_Paulson23}, on the other hand, take a more directed approach, by searching the most promising parameter also outside of the current safe set and then evaluating the most uncertain safe parameter in the neighborhood of the projection of such promising parameter on the boundary of the safe set.
In the context of safe active learning, \citet{li2022safe} consider the safe active learning problem for a multi-output Gaussian process, and choose the parameters to evaluate as those with the highest predictive entropy and high probability of being safe.
All of these methods rely on function uncertainty to drive exploration, while we directly maximize the information gained about the safety of parameters.
In parallel independent work, \citet{huebotter2024informationbased} proposed \textsc{ITL}\xspace, an active learning algorithm that aims at learning the values of a function in a target set $\mathcal{A}$ by actively sampling observations in a sample space $\mathcal{S}$, whose relation to $\mathcal{A}$ can be arbitrary. To do so, \textsc{ITL}\xspace maximizes the mutual information between the prediction targets in $\mathcal{A}$ and the next observation in $\mathcal{S}$. While addressing a more generic problem, when the target space is the set of potential maximizers and the sample space the set of safe parameters, \textsc{ITL}\xspace effectively becomes a safe BO algorithm.

Safe exploration also arises in other contexts, like the one of Markov decision processes (MDP) \citep{safe_mdp_ergodicity,safe_mdp_zoo}, or the one of coverage control \citep{Wang2012_cov_control}. In particular, \citet{turchetta_mdp_2016,turchetta_krause_2019} traverse the MDP to learn about the safety of parameters using methods that, at their core, explore using the same ideas as \safeopt and \stageopt to select parameters to evaluate, while \citet{prajapat_2022} apply these concept to the problem of multi agent coverage control. Consequently, our proposed method for safe exploration is also directly applicable to their setting.

\section{Problem Statement and Preliminaries}\label{sec:ise_problem_statement}

In this section, we introduce the problem and the notation that we use throughout the paper.
We are given an unknown and expensive function $f: \mathcal{X} \rightarrow \mathbb{R}$ that we aim to optimize by sequentially evaluating it at parameters $\bm x_n \in \mathcal{X}$. However, we are not allowed to evaluate parameters that violate some notion of safety. We define the safety of parameters in terms of another a priori unknown and expensive to evaluate safety constraint $s: \mathcal{X} \rightarrow \mathbb{R}$, s.t.\@ parameters that satisfy $s(\bm x) \geq 0$ are classified as safe, while the others as unsafe \footnote{This choice is without loss of generality, since we can incorporate any known non-zero safety threshold $\tau(\bm x)$ in the constraint: $s^\prime(\bm x) \coloneqq s(\bm x) - \tau(\bm x)$.}. To start exploring safely, we also assume to have access to an initial safe parameter $\bm x_0$ that satisfies the safety condition, $s(\bm x_0) \geq 0$. 

Starting from the safe seed $\bm x_0$, we sequentially select safe parameters $\bm x_n \in \mathcal{X}$ to evaluate $f$ and $s$ at. These evaluations at $\bm x_n$ produce noisy observations of the objective $f$ and of the safety constraint $s$: $y_n^f \coloneqq f(\bm x_n) + \nu_n^f$ and $y_n^s \coloneqq s(\bm x_n) + \nu_n^s$, corrupted by additive homoskedastic sub-Gaussian noise $\nu_n^{f/s}$ with zero mean and bounded by $\sigma_\nu^{f/s}$. The goal of an ideal safe optimization algorithm would then be, through these safe evaluations, to solve
\begin{equation}\label{eq:ideal_safe_opt_goal}
\begin{split}
&\max_{\bm x \in \mathcal{X}} f(\bm x) \\
&\text{s.t. } s(\bm x) \geq 0.
\end{split}
\end{equation}
However, it is in general not possible to find the global optimum that  \cref{eq:ideal_safe_opt_goal} implies without ever evaluating unsafe parameters. Instead, we aim to find the safe optimum reachable from the safe seed $\bm x_0$, i.e.\@ the parameter that is reachable from $\bm x_0$ by evaluating only safe parameters and that yields the largest value of $f$ among all such reachable parameters. We illustrate this safe optimization task in \cref{fig:problem_statement_example}, where we highlight in blue the safe area reachable from $\bm x_0$ and the optimum of $f$ within it as the safe optimization goal. In \cref{subsec:largest_reachable_safe_set} we provide more intuition about the largest region that is safely reachable from $\bm x_0$, while in \cref{subsec:definition_of_safe_set} we present a formal definition of it.

\begin{figure}[t] 
    \centering
    \subfloat[Problem components.]{%
        \includegraphics[width=0.47\textwidth]{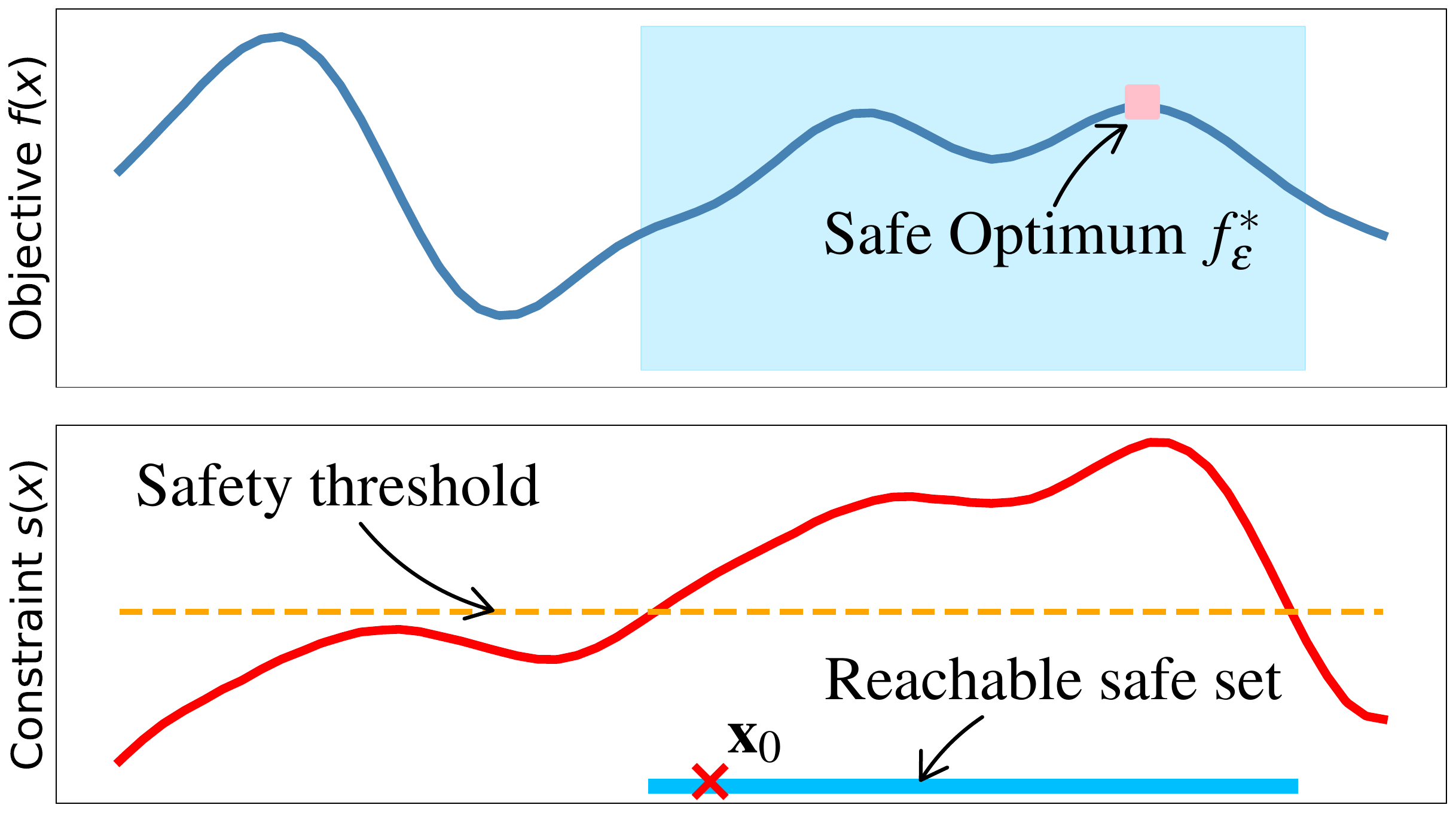}%
        \label{fig:problem_statement_example}%
        }%
    \hfill%
    \subfloat[Mutual information for exploration.]{%
        \includegraphics[width=0.47\textwidth]{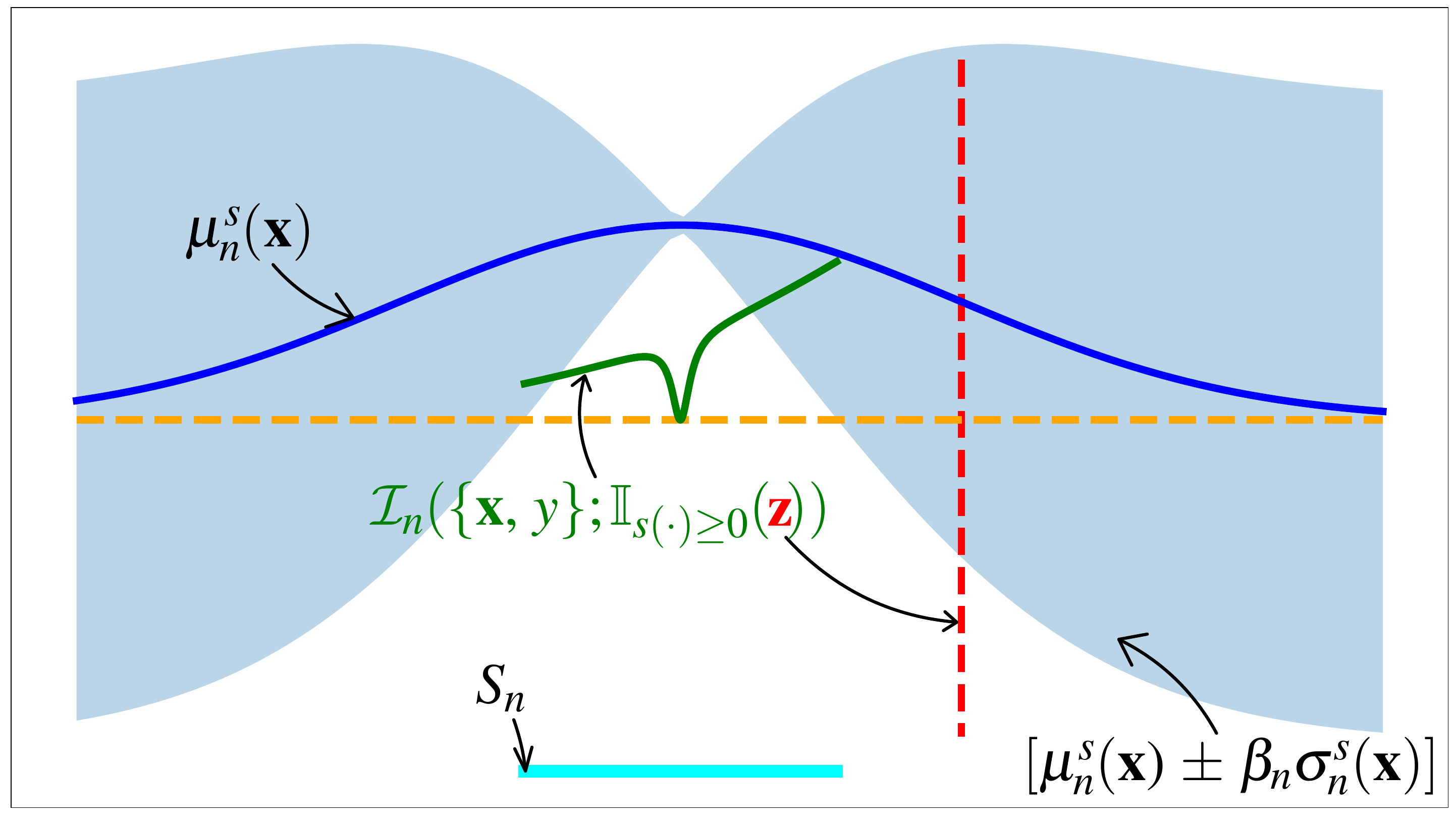}%
        \label{fig:mutual_information}%
        }%
    \caption{In (\subref{fig:problem_statement_example}) we illustrate the safe optimization task. Based on the unknown safety constraint $s$, we are only allowed to evaluate parameters $\bm x$ with values $s(\bm x)$ above the safety threshold (dashed line). Starting from a safe seed $\bm x_0$ a safe optimization strategy needs to find the optimum of the unknown objective $f$ within the largest reachable safe region of the parameter space containing $\bm x_0$. In (\subref{fig:mutual_information}) we show the mutual information $I_n(\{\bm x, y\}; \cpsi(\bm z))$ in green for different parameters $\bm x$ inside the safe set and for a fixed $\bm z$ outside (red dashed line). The exploration part of our algorithm maximizes this quantity jointly over $\bm x$ and $\bm z$.}
    \label{fig:problem_statement_and_mutual_info_example}
    \hfill
\end{figure}

\subsection{Probabilistic model}

As both $f$ and $s$ are unknown and the evaluations $y_n^{f/s}$ are noisy, it is not feasible to select parameters that are safe with certainty. Instead, we provide high-probability safety guarantees by exploiting a probabilistic model for $f$ and $s$.
To this end, we follow \citet{berkenkamp_bayesian_2020} and extend the domain $\mathcal{X}$ to $\tilde{\mathcal{X}} \coloneqq \mathcal{X} \times \{0, 1\}$ and define the function $h: \tilde{\mathcal{X}} \rightarrow \mathbb{R}$ as
\begin{equation}\label{eq:h_definition}
h(\bm x, i) = \begin{cases}
f(\bm x), & \text{if } i = 0 \\
s(\bm x), & \text{if } i = 1.
\end{cases}
\end{equation}
Furthermore, we assume that $h$ has bounded norm in the Reproducing Kernel Hilbert Spaces (RKHS) \citep{Scholkopf2002} $\mathcal{H}_k$ associated to some kernel $k: \tilde{\mathcal{X}} \times \tilde{\mathcal{X}} \rightarrow \mathbb{R}$ with $k\left((\bm x, i), (\bm x', j)\right) \leq 1$: $\norm{h}_{\mathcal{H}_k} = B < \infty$. This assumption allows us to use a Gaussian process (GP) \citep{srinivas_gaussian_2010} as probabilistic model for $h$.

A Gaussian process is a stochastic process specified by a mean function $\mu: \tilde{\mathcal{X}} \rightarrow \mathbb{R}$ and a kernel $k$ \citep{rassmussen_gaussian_2006}. It defines a probability distribution over real-valued functions on $\tilde{\mathcal{X}}$, such that any finite collection of function values at parameters $[(\bm x_1, i_1), \dots, (\bm x_n, i_n)]$ is distributed as a multivariate normal distribution.
The GP prior can then be conditioned on (noisy) function evaluations $\mathcal{D}_n = \{\left((\bm x_j, i_j), y_j\right)\}_{j=1}^n$. If the noise at each observation is Gaussian, $\nu_n \sim \mathcal{N}(0, \sigma_\nu^2)$ $\forall n$, then the resulting posterior is also a GP with posterior mean and variance given by
\begin{equation}
\begin{split}
\mu_n(\bm x, i) &= \mu(\bm x, i) + \bm k(\bm x, i)^\top (\bm K + \bm I \sigma_\nu^2)^{-1}\left(\bm y - \bm \mu \right), \\
\sigma_n^2(\bm x, i)
&= k\left((\bm x, i), (\bm x, i)\right) - \bm k(\bm x, i)^\top (\bm K + \bm I \sigma_\nu^2)^{-1} \bm k(\bm x, i),
\end{split}
\end{equation}
where $\bm \mu \coloneqq [\mu(\bm x_1, i_1), \dots \mu(\bm x_n, i_n)]$ is the mean vector at parameters $(\bm x_j, i_j) \in \mathcal{D}_n$ and $\left[\bm y\right]_j \coloneqq y_j$ the corresponding vector of observations. We have $\left[\bm k(\bm x, i)\right]_j \coloneqq k\left((\bm x, i), (\bm x_j, i_j)\right)$, the kernel matrix has entries $\left[\bm K\right]_{lm} \coloneqq k\left((\bm x_l, i_l), (\bm x_m, i_m)\right)$, and $\bm I$ is the identity matrix. In the following, we assume without loss of generality that the prior mean is identically zero: $\mu(\bm x, i) \equiv 0$.

From the GP that models $h$, it is straightforward to recover the posterior mean and variance for the objective $f$ and the constraint $s$:
\begin{equation}\label{eq:separate_posterior_f_s}
\begin{split}
& \mu_n^f(\bm x) = \mu_n(\bm x, 0), ~~ \mu_n^s(\bm x) = \mu_n(\bm x, 1), \\
& \sigma_n^f(\bm x) = \sigma_n(\bm x, 0), ~~ \sigma_n^s(\bm x) = \sigma_n(\bm x, 1).
\end{split}
\end{equation}

\subsection{Safe set}

Using the previous assumptions, we can construct high-probability confidence intervals on the safety constraint values $s(\bm x)$. Concretely, \citet{kernelized_bandits} show that, for any $\delta > 0$, it is possible to find a sequence of positive numbers $\{\beta_n\}$ such that the true function $h(\bm x, i)$ is within the confidence interval $\left[\mu_n(\bm x, i) \pm \beta_n\sigma_n(\bm x, i)\right]$ with probability at least $1 - \delta$, jointly for all $(\bm x, i) \in \mathcal{\tilde{X}}$ and $n \geq 1$. As \citet{kernelized_bandits} show, a possible choice for such sequence is:
\begin{equation}\label{eq:beta_def}
\beta_n \coloneqq B + R\sqrt{2\left(\ln(e/\delta) + \gamma_n\right)},
\end{equation}
where $R$ is an upper bound on the sub-Gaussian observation noise (i.e.\@ such that both $\mathbb{P}\left\{|\nu_n^s| \geq \lambda\right\} \leq 2 \exp\{-\lambda^2/R^2\}$ and $\mathbb{P}\left\{|\nu_n^f| \geq \lambda\right\} \leq 2 \exp\{-\lambda^2/R^2\}$), and $\gamma_n$ is the maximum information capacity of the chosen kernel after $n$ iterations: $\gamma_n = \max_{|D| = n}I\left(\bm h(D); \bm y(D)\right)$ \citep{srinivas_gaussian_2010,gp_opt_with_mi}.
From these confidence intervals, using \cref{eq:separate_posterior_f_s} we can derive analogous intervals for $s$, $s(\bm x) \in \left[\mu_n^s(\bm x) \pm \beta_n\sigma_n^s(\bm x)\right]$, which we use to define the notion of a \textit{safe set}
\begin{equation}\label{eq:safe_set_definition}
S_n \coloneqq \{\bm x \in \mathcal{X} : \mu_n^s(\bm x) - \beta_n\sigma_n^s(\bm x) \geq 0\} \cup S_{n - 1};~~S_0 \doteq \bm x_0,
\end{equation}
namely the set that contains all parameters whose $\beta_n$-lower confidence bound is above the safety threshold, together with all parameters that were in previous safe sets. Thanks to the set union in the definition of $S_n$, at each iteration the safe set either expands or remains unchanged, guaranteeing that the safe seed $\bm x_0$ is always contained in it. As a consequence of this definition, we also know that all parameters in $S_n$ are safe, $s(\bm x) \geq 0$ for all $\bm x \in S_n$, with probability of at least $1 - \delta$ jointly over all iterations $n$:

\begin{restatable}{lemma}{SafeSetHighProbab}
\label{lemma:safe_set_is_safe_with_high_probability}
Choose $\delta \in (0, 1)$ and let $h$ as defined in \cref{eq:h_definition} be an element in the RKHS $\mathcal{H}_k$ associated to some kernel $k$, with $\norm{h}_{\mathcal{H}_k} = B < \infty$. Let, moreover, the sequence of positive numbers $\{\beta_n\}$ be as in \cref{eq:beta_def}. Then if we define the safe set as in \cref{eq:safe_set_definition}, we have that $s(\bm x) \geq 0$ with probability of at least $1 - \delta$, jointly for all $\bm x \in S_n$ and for all $n$.
\end{restatable}

\begin{proof}
See \cref{appendix:ise_proofs}.
\end{proof}

\subsection{Largest reachable safe set and safe optimum}\label{subsec:largest_reachable_safe_set}
At the beginning of this section, we stated that the goal of a safe optimization algorithm is to find the optimum within the largest safe set reachable from the safe seed $\bm x_0$, like the one shown in \cref{fig:problem_statement_example}. Therefore, we need to define what we mean with \textit{largest reachable safe set}. 

Intuitively, given that we can only evaluate parameters that are safe with high probability and that at the beginning we only know that $\bm x_0$ is safe, we need a way to infer what parameters are also safe, given that $\bm x_0$ is and assuming that we know the value $s(\bm x_0)$ up to a precision of $\varepsilon$. Then we can repeat this process recursively, until we cannot add any more parameters. 

A useful tool for the theoretical analysis that allows us formalize such intuition is a one-step \emph{expansion operator} $R_\varepsilon$, such that $R_\varepsilon(\bm x_0)$ consists of all the parameters that we can assume safe knowing $s(\bm x_0)$ up to $\varepsilon$, $R_\varepsilon\left(R_\varepsilon(\bm x_0)\right)$ the set of all parameters that we can assume safe if we know the value of $s(\bm x)$ up to $\varepsilon$ for all $\bm x$ in $R_\varepsilon(\bm x_0)$, and so on.
With $R_\varepsilon$ at hand, it becomes immediate to define the largest reachable safe set $S_\varepsilon(\bm x_0)$ as the set that we obtain by recursively applying $R_\varepsilon$ to $\bm x_0$ an infinite number of times: $S_\varepsilon(\bm x_0) \coloneqq \lim_{n\to\infty} R^n_\varepsilon(\bm x_0)$.

For their \safeopt algorithm, \citet{sui_safe_2015} define such an operator leveraging the Lipschitz assumption on the safety constraint. Namely, starting from an initial set of parameters known to be safe, they recursively add to that set those parameters that must also be safe given the Lipschitz condition, until no new parameters can be added to the set. 
In \cref{sec:ise_theory} of this paper, for the generalization from the current set of safe parameters, we can rely solely on the GP posterior, so that a natural choice for the largest safe region reachable from the safe seed is the set that is safely reachable from $\bm x_0$ by all well behaved GPs. This intuition leads us, in \cref{subsec:definition_of_safe_set}, to the formal definition of a corresponding expansion operator $R_\varepsilon$ and the resulting reachable safe set $S_\varepsilon(\bm x_0)$. 

Once we have identified the largest safe region of the domain that we can safely reach starting from the safe seed, $S_\varepsilon(\bm x_0)$, we can naturally define also the optimum value of $f$ within that region:
\begin{definition}[Safe optimum $f^*_\varepsilon$]\label{def:safe_optimum}
Let $\varepsilon > 0$, we define the safe optimum reachable from $\bm x_0$ as:
\begin{equation}
f^*_\varepsilon \coloneqq \max_{\bm x \in S_\varepsilon(\bm x_0)}f(\bm x).
\end{equation}
\end{definition}
The reachable safe optimum $f^*_\varepsilon$ is, therefore, the natural objective for the safe optimization problem that we consider in this paper. We present more details and a formal definition of $S_\varepsilon(\bm x_0)$ in \cref{subsec:definition_of_safe_set}.

\subsection{Bayesian optimization}\label{sec:mes_intro}

If the set $S_\varepsilon(\bm x_0)$ were known, then the problem would be global optimization within $S_\varepsilon(\bm x_0)$. A popular framework to find the global optimum of an unknown and expensive function via sequential evaluations is Bayesian optimization. Given a function $f$ defined on some fixed domain $\Omega$, BO algorithms find the $\argmax_{\bm x \in \Omega}f(\bm x)$ by building a probabilistic model of $f$ and evaluating parameters that maximize a so called \textit{acquisition function} $\alpha: \Omega \rightarrow \mathbb{R}$ defined via that model. The probabilistic model is then updated in a Bayesian fashion using the evaluation results.
\cref{alg:bo_generic} shows pseudo-code for the generic BO loop.

BO algorithms differ for the acquisition function that they use. As mentioned in \cref{sec:ise_introduction}, one of such acquisition function that is based on an information-theoretic criterion is the Max-value Entropy Search (MES) acquisition function \citep{wang_max-value_2017}.
More specifically, the MES acquisition function computes the mutual information between an evaluation at parameter $\bm x$ with observed value $y$ and the objective function's optimum value: $I_n\left(\{\bm x, y\}; f^*\right)$, with $f^* \coloneqq \max_{\bm x}f(\bm x)$, so that the next parameter to evaluate is the most informative about the value of the optimum. 
In the case of a GP prior and Gaussian additive noise, the MES mutual information can be expressed as
\begin{equation}\label{eq:mes_acquisition_expression}
I_n\left(\{\bm x, y\}; f^*\right) = \mathbb{E}_{f^*}\left[\frac{\theta_{f^*}(\bm x)\psi(\theta_{f^*}(\bm x))}{2\Psi(\theta_{f^*}(\bm x))} - \ln\left(\Psi(\theta_{f^*}(\bm x))\right)\right],
\end{equation}
where $\theta_{f^*}(\bm x) = \frac{f^* - \mu_n(\bm x)}{\sigma_n(\bm x)}$, while $\psi$ and $\Psi$ are, respectively, the probability density function of a standard Gaussian and its cumulative distribution.
The expectation in \cref{eq:mes_acquisition_expression} is over the possible values of the optimum and can be approximated via a Monte Carlo estimate. \citet{wang_max-value_2017} have shown that estimating \cref{eq:mes_acquisition_expression} even with a single properly chosen sample achieves vanishing regret with high probability, meaning that, with high probability, it leads to learning about the optimum value up to arbitrary confidence.

\begin{algorithm}[t]
   \caption{Generic BO loop}
   \label{alg:bo_generic}
\begin{algorithmic}[1]
   \STATE {\bfseries Input:} GP prior for objective function $f$ $\sim$ GP($\mu_0$, $k$, $\sigma_\nu$)
   \FOR{$n=0$, \dots, $N$}
   \STATE $\bm x_{n+1}$ $\leftarrow \argmax_{\bm x} \alpha(\bm x)$
   \STATE $y_{n+1}$ $\leftarrow$ $f(\bm x_{n+1}) + \nu_n$
   \STATE Update GP posteriors with $(\bm x_{n+1}, y_{n+1})$
   \ENDFOR
\end{algorithmic}
\end{algorithm}

\section{Safe Optimization Algorithm}\label{sec:safe_optimization_algorithm}

In \cref{subsec:largest_reachable_safe_set}, we defined the goal of a safe optimization algorithm, namely finding the safe optimum $f^*_\varepsilon$, while in \cref{sec:mes_intro}, we introduced BO as a framework to solve unconstrained global optimization. To solve our safe optimization problem, we need to design an acquisition function $\alpha$ that at each iteration $n$ we can use to identify the next parameter to evaluate by maximizing it within the current safe set:
\begin{equation}
\bm x_{n+1} \in \argmax_{\bm x \in S_n} \alpha(\bm x).
\end{equation}

A key element to keep in mind when designing a BO acquisition function is the exploration-exploitation trade-off, as the optimum could be both in regions we know more about and seem promising, as well as in areas where the uncertainty about the function values is high. In safe BO, the fact that we are only allowed to evaluate parameters that are safe with high probability adds a further layer of complexity to the exploration-exploitation problem, since applying exploration strategies of popular BO algorithms only within the set of safe parameters is generally insufficient to drive the expansion of the safe set and learn about the safe optimum, as for example shown by \citet{sui_safe_2015} for the well known GP-UCB algorithm \citep{srinivas_gaussian_2010}. In \cref{fig:toy_example}, we can see an example of such situations, where the \mes acquisition \citep{wang_max-value_2017} constrained to the safe set fails to expand the safe set sufficiently to uncover the true safe optimum.

To overcome such issues, a safe BO algorithm needs, therefore, to simultaneously learn about the safety of parameters outside the current safe sets, promoting its expansion, and explore-exploit within the safe set to learn about the safe optimum. 
In order to achieve this behavior, in \cref{sec:safe_exploration} we propose an acquisition function that selects the parameters to evaluate as the ones that maximize the information gain about the safety of parameters outside of the safe set. We then combine it with the \mes acquisition introduced in \cref{sec:mes_intro}, and show that this combinations converges to the safe optimum in \cref{subsec:convergence_results}.

\begin{figure}[t] 
    \centering
    \includegraphics[width=\textwidth]{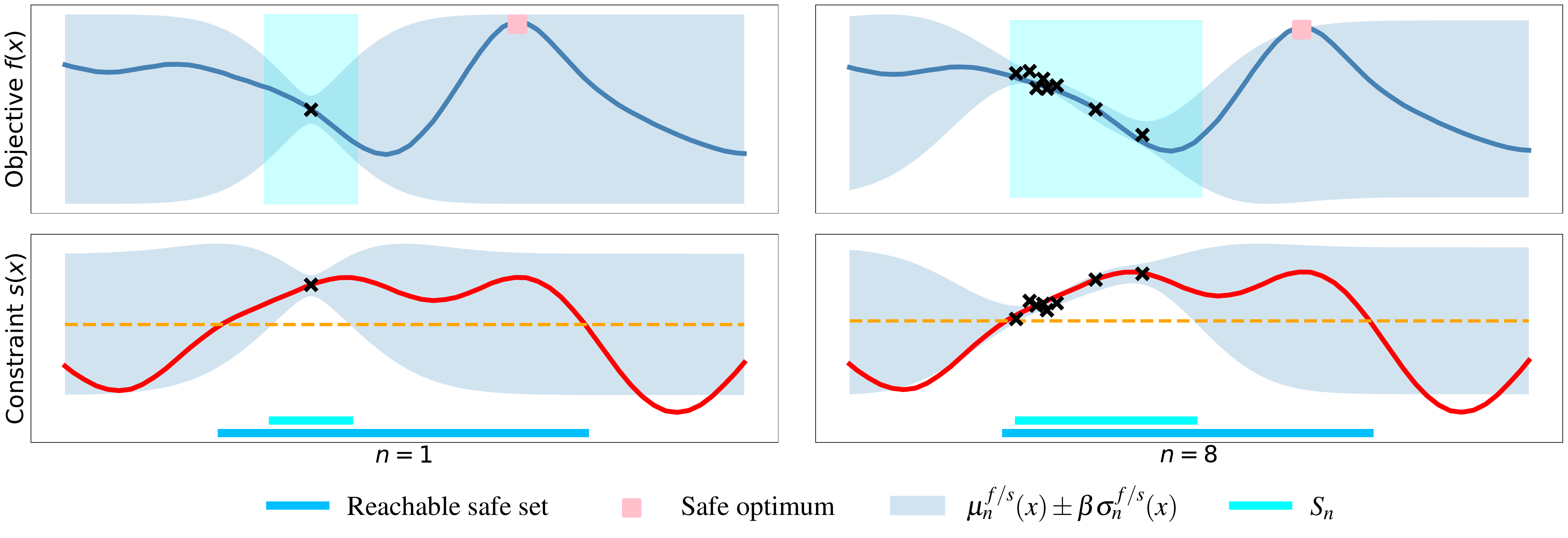}
    \caption{1D example that illustrates why an explicit exploration component that promotes the expansion of the safe set is crucial in safe BO. Here the function in the top half, $f$, is the unknown objective to optimize, while the lower half shows the unknown safety constrain, $s$. The algorithm that generated the plots in the figure chooses the next parameter to evaluate as the one that maximizes the pure \mes acquisition function for $f$, constrained within the current safe set $S_n$. We can see that the safe optimum (pink square) is outside the initial safe set (left plots), but the acquisition function has no interest in evaluating points that could expand the safe set further on the right side of the safe set, since the right boundary of the safe set has clearly a negligible probability of containing the optimum. Instead, the \mes acquisition function focuses on the left boundary of the safe set - as we can see in the plots on the right - since that is the likely location of the optimum within the current safe set.}
    \label{fig:toy_example}
    \hfill
\end{figure}

\subsection{Safe exploration}\label{sec:safe_exploration}

To be able to expand the safe set beyond its current extension, we need to learn what parameters outside of the safe set are also safe with high probability. Most existing safe exploration methods rely on uncertainty sampling over subsets within $S_n$. The safe exploration strategy that we propose instead uses an information gain measure to identify the parameters that allow us to efficiently learn about the safety of parameters \emph{outside} of $S_n$. In other words, we want to evaluate $s$ at safe parameters that are maximally informative about the safety of other parameters, in particular of those where we are uncertain about whether they are safe or not.

To this end, we need a corresponding measure of information gain. We define such a measure using the binary variable $\cpsi(\bm x)$, which is equal to one iff $s(\bm x) \geq 0$ and zero otherwise. Its entropy is given by
\begin{equation}\label{eq:exact_psi_entropy}
H_n\left[\cpsi(\bm z)\right] = -p_n^-(\bm z) \ln(p_n^-(\bm z)) - \left(1 - p_n^-(\bm z)\right) \ln(1 - p_n^-(\bm z)),
\end{equation}
where $p_n^-(\bm z)$ is the probability of $\bm z$ being unsafe: $p_n^-(\bm z) = \frac{1}{2} + \frac{1}{2}\erf \left(-\frac{1}{\sqrt{2}}\frac{\mu^s_n(\bm z)}{\sigma^s_n(\bm z)}\right)$. The random variable $\cpsi(\bm z)$ has high-entropy where we are uncertain whether a parameter is safe or not; that is, its entropy decreases monotonically as $|\mu^s_n(\bm z)|$ increases and the GP posterior moves away from the safety threshold. It also decreases monotonically as $\sigma^s_n(\bm z)$ decreases and we become more certain about whether the constraint is violated or not. This behavior also implies that the entropy goes to zero as the confidence about the safety of $\bm z$ increases, as desired.

Given $\cpsi$, we consider the mutual information $I\left(\{\bm x, y\}; \cpsi(\bm z)\right)$ between an observation $y$ at a parameter $\bm x$ and the value of $\cpsi$ at another parameter $\bm z$. Since $\cpsi$ is the indicator function of the safe regions of the parameter space, the quantity $I_n\left(\{\bm x, y\}; \cpsi(\bm z)\right)$ measures how much information about the safety of $\bm z$ we gain by evaluating the safety constraint $s$ at $\bm x$ at iteration $n$, averaged over all possible observed values $y$. This interpretation follows directly from the definition of mutual information:
\begin{equation}
I_n\left(\{\bm x, y\}; \cpsi(\bm z)\right) = H_n\left[\cpsi(\bm z)\right] - \mathbb{E}_{y}\left[H_{n+1}\left[\cpsi(\bm z) \middle| \{\bm x, y\}\right]\right],
\end{equation}
where $H_n[\cpsi(\bm z)]$ is the entropy of $\cpsi(\bm z)$ according to the GP posterior at iteration $n$, while $H_{n+1}\left[\cpsi(\bm z) \middle| \{\bm x, y\}\right]$ is its entropy at iteration $n+1$, conditioned on a measurement $y$ at $\bm x$ at iteration $n$. Intuitively, $I_n\left(\{\bm x, y\}; \cpsi(\bm z)\right)$ is negligible whenever the confidence about the safety of $\bm z$ is high or, more generally, whenever an evaluation at $\bm x$ does not have the potential to substantially change our belief about the safety of $\bm z$. The mutual information is large whenever an evaluation at $\bm x$ on average causes the confidence about the safety of $\bm z$ to increase significantly. As an example, in \cref{fig:mutual_information} we plot $I_n\left(\{\bm x, y\}; \cpsi(\bm z)\right)$ as a function of $\bm x \in S_n$ for a specific choice of $\bm z$ and for an RBF kernel. As one would expect, we see that the closer $\bm x$ gets to $\bm z$, the bigger the mutual information becomes, and that it vanishes in the neighborhood of previously evaluated parameters, where the posterior variance is small.

To compute $I_n\left(\{\bm x, y\}; \cpsi(\bm z)\right)$, we need to average \cref{eq:exact_psi_entropy} conditioned on an evaluation $y$ over all possible values of $y$. However, the resulting integral is intractable given the expression of $H_n[\cpsi(\bm z)]$ in \cref{eq:exact_psi_entropy}. In order to get a tractable result, we derive a close approximation of \cref{eq:exact_psi_entropy},
\begin{equation}\label{eq:approximated_psi_entropy}
H_n\left[\cpsi(\bm z)\right] \approx \hat{H}_n\left[\cpsi(\bm z)\right] \doteq \ln(2) \exp\left\{-\frac{1}{\pi\ln(2)}\left(\frac{\mu^s_n(\bm z)}{\sigma^s_n(\bm z)}\right)^2\right\}.
\end{equation}
We obtained the approximation in \cref{eq:approximated_psi_entropy} by truncating the Taylor expansion of $H_n[\cpsi(\bm z)]$ at the second order, and noticing that it recovers almost exactly its true behavior (see \cref{appendix:entropy_approx} for details).
Since the posterior mean at $\bm z$ after an evaluation at $\bm x$ depends linearly on $\mu_n(\bm x)$, and since the probability density of $y$ depends exponentially on $-\mu_n^2(\bm x)$, using the approximation \cref{eq:approximated_psi_entropy} reduces the conditional entropy $\mathbb{E}_{y}\left[\hat{H}_{n+1}\left[\cpsi(\bm z) \middle| \{\bm x, y\}\right]\right]$ to a Gaussian integral with the exact solution
\begin{equation}\label{eq:averaged_post_measurement_entropy}
\begin{split}
\mathbb{E}_{y}&\left[\hat{H}_{n+1}\left[\cpsi(\bm z) \middle| \{\bm x, y\}\right]\right] =\\
&\ln(2)\sqrt{\frac{\sigma_\nu^2 + \sigma_n^2(\bm x)(1 - \rho_n^2(\bm x, \bm z))}{\sigma_\nu^2 + \sigma_n^2(\bm x)(1 + c_2\rho_n^2(\bm x, \bm z))}}\exp\left\{-c_1\frac{\mu_n^2(\bm z)}{\sigma_n^2(\bm z)}\frac{\sigma_\nu^2 + \sigma_n^2(\bm x)}{\sigma_\nu^2 + \sigma_n^2(\bm x)(1 + c_2\rho_n^2(\bm x, \bm z))}\right\},
\end{split}
\end{equation}
where we omitted the $s$ superscript for $\mu_n$ and $\sigma_n$, $\rho_n(\bm x, \bm z)$ is the linear correlation coefficient between $s(\bm x)$ and $s(\bm z)$, and $c_1$ and $c_2$ are given by $c_1 \coloneqq 1/\ln(2)\pi$ and $c_2 \coloneqq 2c_1 - 1$. This result allows us to analytically calculate the approximated mutual information 
\begin{equation}\label{eq:approx_exploration_mutual_info}
\hat{I}_n\left(\{\bm x, y\}; \cpsi(\bm z)\right) \doteq \hat{H}_n\left[\cpsi(\bm z)\right] - \mathbb{E}_{y}\left[\hat{H}_{n+1}\left[\cpsi(\bm z) \middle| \{\bm x, y\}\right]\right].
\end{equation}
Now that we have defined a way to measure and compute the information gain about the safety of parameters, we can use it to design the exploration component of our acquisition function, which drives the expansion of the safe set. The natural choice for this component is a function that selects the parameter that maximizes the information gain; that is, we define the \ourmethodlongexp (\ourmethodexp) acquisition $\alpha^{\ourmethodexp}$, as

\begin{equation}\label{eq:alpha_ise}
\alpha^{\ourmethodexp} \coloneqq \max_{\bm z \in \mathcal{X}}\hat{I}_n\left(\{\bm x, y\}; \cpsi(\bm z)\right).
\end{equation}
Evaluating $s$ at $\bm x \in \argmax_{\bm x \in S_n}\alpha^{\ourmethodexp}(\bm x)$ maximizes the information gained about the safety of some parameter $\bm z \in \mathcal{X}$, so that it allows us to efficiently learn about parameters that are not yet known to be safe. While $\bm z$ can lie in the whole domain, the parameters where we are the most uncertain about the safety constraint lie outside the safe set. By leaving $\bm z$ unconstrained, we show in our theoretical analysis in \cref{sec:ise_theory} that, once we have learned about the safety of parameters outside the safe set, \cref{eq:alpha_ise} resorts to learning about the constraint function also inside $S_n$. In \cref{sec:ise_theory} we also show that this behavior yields exploration guarantees, meaning that it leads to classify the whole $S_\varepsilon(\bm x_0)$ as safe.

\begin{algorithm}[t]
   \caption{\ourmethodlong}
   \label{alg:algorithm}
\begin{algorithmic}[1]
   \STATE {\bfseries Input:} GP priors for $h$ ($\mu_0$, $k$, $\sigma_\nu$) and safe seed $\bm x_0$
   \FOR{$n=0$, \dots, $N$}
   \STATE $\bm x_{n+1}$ $\leftarrow \argmax_{\bm x \in S_n}~\max\left\{\alpha^{\ourmethodexp}(\bm x), \alpha^{\mes}(\bm x)\right\}$
   \STATE $y_{n+1}^f$, $y_{n+1}^s$ $\leftarrow$ $f(\bm x_{n+1}) + \nu_n^f$, $s(\bm x_{n+1}) + \nu_n^s$
   \STATE Update GP posteriors with $(\bm x_{n+1}, y_{n+1}^f)$ and $(\bm x_{n+1}, y_{n+1}^s)$ 
   \ENDFOR
\end{algorithmic}
\end{algorithm}

\subsection{Optimization}
Next, we address exploration and exploitation within the safe set, in order to learn about the current safe optimum.
In the previous subsection, we discussed an exploration strategy that promotes the expansion of the safe set beyond the parameters that are currently classified as safe with high probability. As discussed, such an exploration component is essential for the task at hand, especially in those circumstances where the location of the safe optimum is far away from the safe seed $\bm x_0$. In order to find the safe optimum $f^*_\varepsilon$, however, it is not sufficient to learn about the largest reachable safe set, but one also needs to learn about the function values within the safe set, until the optimum can be identified with high confidence. 

To this end, we can pair the mutual information $\hat{I}_n\left(\{\bm x, y\}; \cpsi(\bm z)\right)$ with a quantity that measures the information gain about the safe optimum. For such a pairing to work, we need a quantity that has the same units as $\alpha^{\ourmethodexp}$ to measure progress. A natural choice for such quantity is the mutual information between an evaluation at $\bm x$ and the safe optimum: $I_n\left(\{\bm x, y\}; f^*_\varepsilon\right)$. With this choice, the next parameter to evaluate, $\bm x_{n + 1}$, achieves the maximum between the two information gain measures:
\begin{equation}\label{eq:combined_acqusisition_1}
\bm x_{n + 1} \in \argmax_{\bm x \in S_n}~\max\left\{\alpha^{\ourmethodexp}(\bm x), \alpha^{\mes}(\bm x)\right\},
\end{equation}
where $\alpha^{\mes} \coloneqq I_n\left(\{\bm x, y\}; f^*_\varepsilon\right)$ is the MES acquisition function that \citet{wang_max-value_2017} proposed, as explained in \cref{sec:mes_intro}.
By selecting $\bm x_{n+1}$ according to \cref{eq:combined_acqusisition_1}, we choose the parameter that yields the highest information gain about the two objectives we are interested in: what parameters outside the current safe set are also safe and what is the optimum value within the current safe set. Indeed, as we show in \cref{sec:ise_theory}, with high probability this selection criterion eventually leads to classify as safe the whole $S_\varepsilon(\bm x_0)$ and to identify $f^*_\varepsilon$ with high confidence.

\begin{figure}[t] 
    \centering
    \subfloat[$n = 2$]{%
        \includegraphics[width=0.47\textwidth]{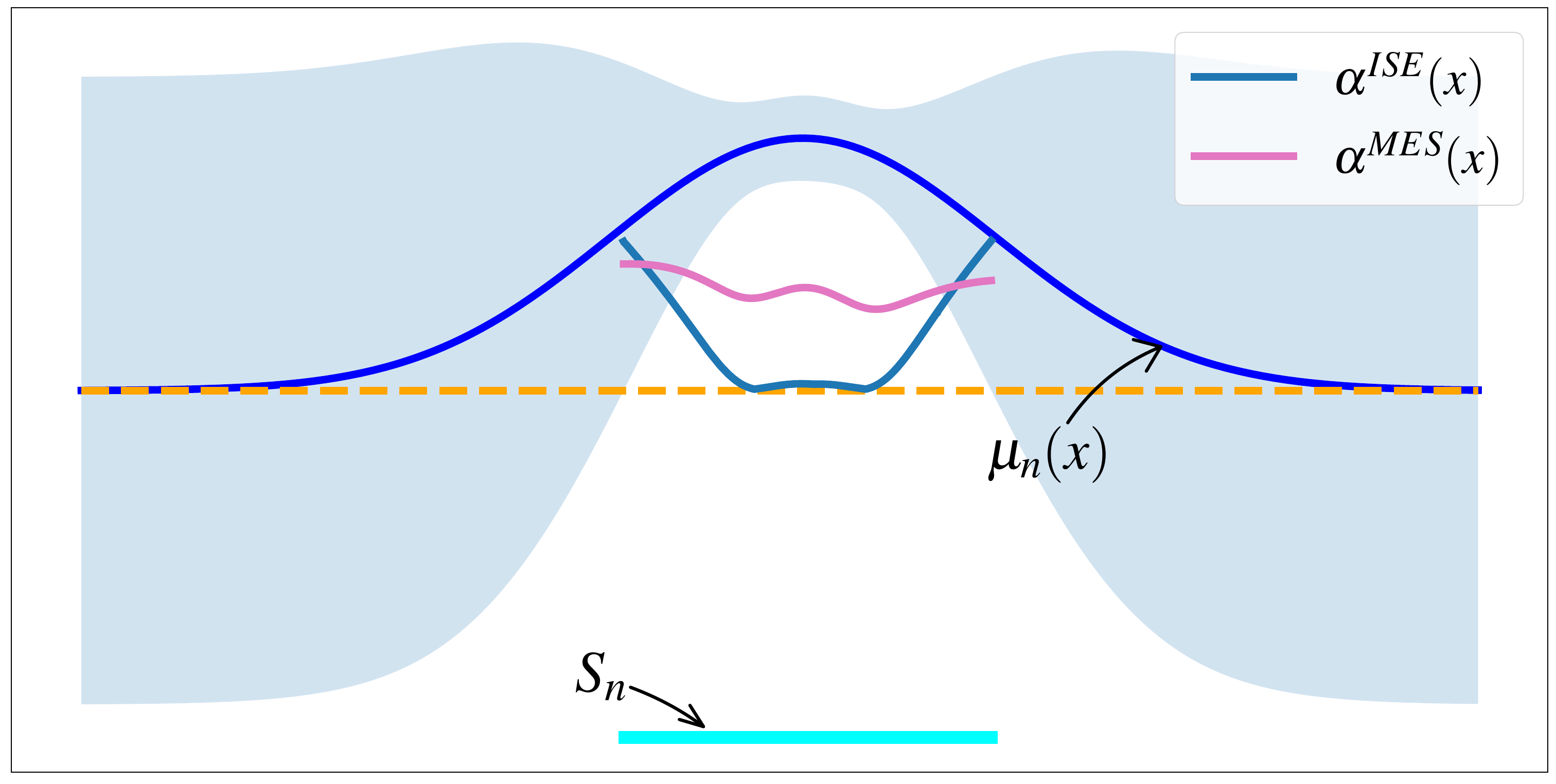}%
        \label{fig:acquisition_example_1}%
        }%
    \hfill
    \subfloat[$n = 7$]{%
        \includegraphics[width=0.47\textwidth]{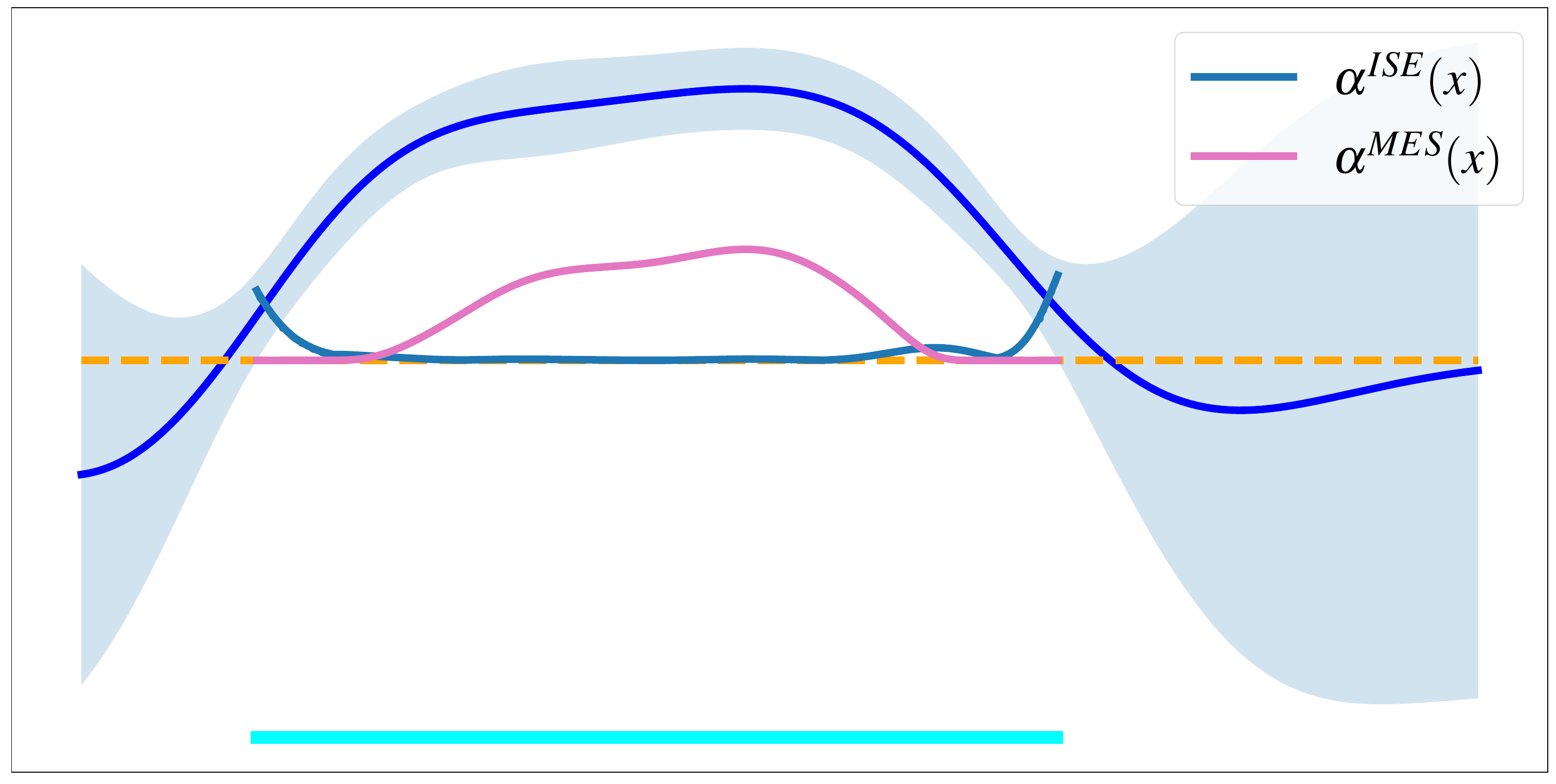}%
        \label{fig:acquisition_example_2}%
        }%
    \caption{Example of the two components of the acquisition function. For simplicity, here the objective $f$ and the constraint $s$ are the same function. The two plots show the values of the safe set expansion component ($\alpha^{\ourmethodexp}$) and the optimization component ($\alpha^{\mes}$) of the acquisition function inside the safe set ($S_n$), against the GP posterior mean and confidence interval (the blue curve and shaded area, respectively), with the safety threshold $s(\bm x) = f(\bm x) = 0$ indicated by the orange dashed line. At early stages (\subref{fig:acquisition_example_1}) we see that both components achieve their maximum on the boundary of the safe set, since, according to the GP posterior, it is both promising as a region that contains the optimum and as a region that can give us much information about the safety of parameters outside $S_n$. The plot in (\subref{fig:acquisition_example_2}), however, shows that, as we continue sampling on the border, $\alpha^{\mes}$ vanishes here, since it is unlikely that the optimum is in these neighborhoods. On the contrary, $\alpha^{\ourmethodexp}$ remains non negligible on the boundaries, as these parameters can still give us information about the safety constraint outside of the current $S_n$, possibly leading to its expansion.}
    \label{fig:acquisition_functions_example}
    \hfill
\end{figure}

\cref{fig:acquisition_functions_example} shows an example of the two components of the acquisition function \cref{eq:combined_acqusisition_1}, in a case where the objective and constraint are the same function. We see in \cref{fig:acquisition_example_2} that the exploration component, $\alpha^{\ourmethodexp}$, vanishes within the safe set, where we have high confidence about the safety of $f$, while it achieves its biggest values on the boundary of the safe set, where an observation is most likely to give us valuable information about the safety of parameters about whose safety we don't know much yet. On the contrary, the optimization component of the acquisition function, $\alpha^{\mes}$, has its optimum in the inside of $S_n$, where the current safe optimum of $f$ is likely to be, so that an observation there would increase the amount of information we have about the safe optimum value $f^*$ within $S_n$. It is also possible to show that both acquisition components are bounded by a monotonically increasing function of the posterior variance, meaning that they vanishes in regions where the posterior variance is very small and where, therefore, we are highly confident about the values of $f$ and $s$:
\begin{restatable}{lemma}{MIsDecreaseWithSigma}
\label{lemma:mutual_infos_decrease_with_sigma}
Let $\alpha^{\ourmethodexp}$ be as defined in \cref{eq:alpha_ise} and $\alpha^{\mes} = I_n\left(\{\bm x, y\}; f^*_\varepsilon\right)$, computed using, respectively, the posterior GP for $s$ and for $f$. Then it holds that $\max\left\{\alpha^{\ourmethodexp}(\bm x), \alpha^{\mes}(\bm x)\right\} \leq \max\left\{\frac{\sigma_{s,n}^2(\bm x)}{\sigma_{s,\nu}^2}, \frac{\sigma_{f,n}^2(\bm x)}{2\sigma_{f,\nu}^2}\right\}$.
\end{restatable}
\begin{proof}
See \cref{appendix:ise_proofs}.
\end{proof}

\cref{fig:acquisition_functions_example} also serves as a further intuition for why the single components alone would ultimately lead to an inefficient safe optimization behavior: the exploration part keeps sampling close to the border of $S_n$, until the safe set cannot be expanded further, and then it tries to reduce the uncertainty on the border to very low values, before turning its attention to the inside of the safe set. On the other hand, the exploration-exploitation component within $S_n$ drives the expansion of the safe set only until it is still plausible that the optimum is in the vicinity of the boundary of $S_n$, so that an observation in that area can give us information about the optimum value. As soon as that is not the case anymore, as in \cref{fig:acquisition_example_2}, this component of the acquisition function would keep focusing on the inside of $S_n$ for a large number of iterations, effectively stopping the expansion of the safe set, even in the case that the current safe set is not yet the largest one and it may not yet contain the true safe optimum $f^*_\varepsilon$.

Finally, in \cref{fig:algo_Example} we report an example run of \cref{alg:algorithm}. This figure shows how the acquisition function \cref{eq:combined_acqusisition_1} alternates between evaluating parameters that lead to an expansion of the safe set and parameters that give information about the safe optimum, until it confidently identifies $f^*_\varepsilon$.
In the next section, we analyze the theoretical convergence properties of the acquisition function \cref{eq:combined_acqusisition_1} and show that it does indeed allow us to eventually learn about the safe optimum $f^*_\varepsilon$ with high probability.

\begin{figure}[t] 
    \centering
    \includegraphics[width=\textwidth]{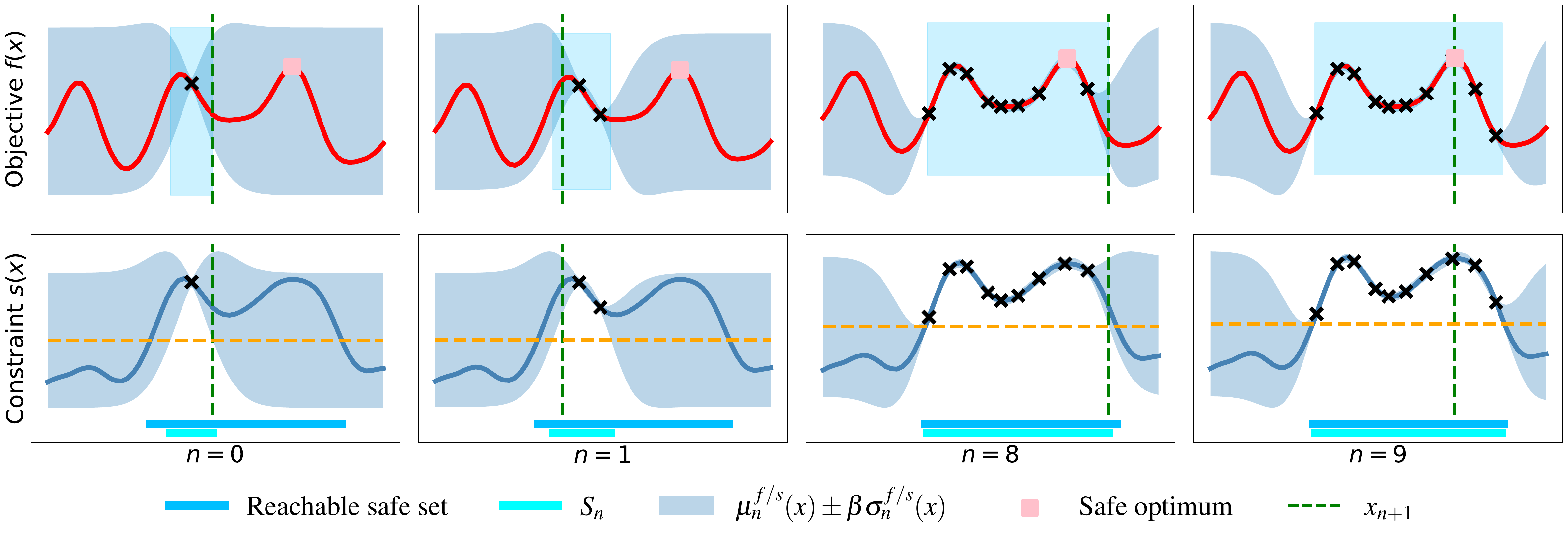}
    \caption{1D example that illustrates \cref{alg:algorithm}. The top row shows the objective function $f$ with the corresponding posterior GP confidence interval at various iterations, while the bottom row does the same for the safety constraint $s$. The dashed orange line represents the safety threshold, while the black crosses mark the observed values $y_n^{f/s}$ at the evaluated parameters $\bm x_n$. We see how the the acquisition function \cref{eq:combined_acqusisition_1} selects the next parameter $\bm x_{n+1}$ (green dashed line) alternating between parameters that are likely to expand the safe set ($n=0$, $n=8$) and parameters that are  informative about the current safe optimum ($n=1$, $n=9$), until it is able to identify the reachable safe optimum (pink square).}
    \label{fig:algo_Example}
    \hfill
\end{figure}

\section{Theoretical Results}\label{sec:ise_theory}

In \cref{subsec:largest_reachable_safe_set}, we provided an intuition for what we defined as the largest reachable safe set $S_\varepsilon(\bm x_0)$, which defines the safe optimization objective $f^*_\varepsilon$, and in \cref{sec:safe_optimization_algorithm} we introduced the selection criterion \cref{eq:combined_acqusisition_1} for our proposed safe BO algorithm, which \cref{alg:algorithm} summarizes.
In this section, we first present a formal definition of the set $S_\varepsilon(\bm x_0)$, in \cref{subsec:definition_of_safe_set}, while in \cref{subsec:convergence_results} we analyze the properties of the selection criterion \cref{eq:combined_acqusisition_1}, in order to show that it leads to the discovery of the largest reachable safe set $S_\varepsilon(\bm x_0)$ and of the optimum value of $f$ within this set.

\subsection{Definition of \texorpdfstring{$S_\varepsilon(\bm x_0)$}{TEXT}}\label{subsec:definition_of_safe_set}

As we explained in \cref{subsec:largest_reachable_safe_set}, the set $S_\varepsilon(\bm x_0)$ is the set of all parameters that are safe with high probability and that are safely reachable from $\bm x_0$ by all well behaved GPs. In order to formally explain what we mean by it, we first need some preliminary definitions. In the following, we refer to a GP $\sim \mathcal{N}(\mu_0, k)$ conditioned on some observation data $\mathcal{D} = \{(\bm x_n, y_n)\}$ as GP$_\mathcal{D}$. 
Using this notation, we can define the set of all well behaved GPs that have small uncertainty over a given set $\Omega$:
\begin{definition}[$\bm \beta$-GP$_s^\varepsilon(\Omega)$]\label{def:k_gp_beta_omega}
Given a kernel $k$, a function $s$, and ${\bm \beta \coloneqq \{\beta_n\}}$ as in \cref{lemma:safe_set_is_safe_with_high_probability}, let $\Omega$ be a subset of the domain and $\varepsilon > 0$. We define the set $\bm \beta$-GP$_s^\varepsilon(\Omega)$ as the one containing all the well behaved GPs -- i.e.\@ those GPs whose $\beta$-confidence interval contains the true value of $s$ -- that have an uncertainty over $\Omega$ smaller than $\varepsilon$:
\begin{equation}\label{eq:k_gp_beta_omega_condition}
\begin{split}
\bm \beta\text{-GP}_s^\varepsilon(\Omega) = \big\{\text{GP}_\mathcal{D}: &~s(\bm x) \in [\mu_n(\bm x) - \beta_n\sigma_n(\bm x), \mu_n(\bm x) + \beta_n\sigma_n(\bm x)] \forall \bm x \in \mathcal{X}, \forall n \leq |\mathcal{D}|\big. \\ 
\big . &\wedge \beta_{|\mathcal{D}|}\sigma_{|\mathcal{D}|}(\bm x) \leq \varepsilon ~ \forall \bm x \in \Omega \big\}.
\end{split}
\end{equation}
\end{definition}
As a consequence of the choice of the sequence $\left\{\beta_n\right\}$, we know that every GP$_\mathcal{D}$ is well behaved with probability at least $1 - \delta$, so that \cref{def:k_gp_beta_omega} is not trivial (see \cref{lemma:in_beta_gp_with_high_prob} in \cref{appendix:ise_proofs} for more details).

Now that we have a notion of well behaved GPs, we need to define what it means for a region of the domain to be safely reachable by all such GPs. Intuitively, a set $\mathcal{Y} \subseteq \mathcal{X}$ is safely reachable by a GP if, starting from the safe seed $\bm x_0$ and progressively reducing the posterior uncertainty over the safe set, eventually $\mathcal{Y}$ is classified as safe.
In order to formalize this intuition, we first define a one step expansion operator, that we can then recursively apply to the safe seed to obtain the largest reachable safe set.
\begin{definition}[Expansion operator $R_\varepsilon$]\label{def:expansion_operator}
Given a dataset $\mathcal{D}$ and a GP conditioned on it, GP$_{\mathcal{D}}$, we call $S_{\text{GP}_\mathcal{D}}$ the safe set associated to such GP posterior, as prescribed by \cref{eq:safe_set_definition}. With this notation, given a safe set $S$ and $\varepsilon > 0$, we define the expansion operator $R_\varepsilon$ as
\begin{equation}\label{eq:expansion_operator}
R_\varepsilon(S) = \big\{\bm x \in \mathcal{X}: \bm x \in S_{\text{GP}_\mathcal{D}} \text{ for all GP } \in \bm \beta\text{-GP}_s^\varepsilon(S) \big\},
\end{equation}
so that $R_\varepsilon(S)$ would be the set of all parameters that are classified as safe, according to \cref{eq:safe_set_definition}, by all well behaved GPs that have an uncertainty of at most $\varepsilon$ over $S$.
\end{definition} 
With the expansion operator $R_\varepsilon$ at hand, we can finally define the largest reachable safe set, as the set that we obtain by recursively applying $R_\varepsilon$ to the safe seed $\bm x_0$:
 
\begin{definition}[Largest reachable safe set $S_\varepsilon(\bm x_0)$]\label{def:largest_safe_set}
Given the expansion operator $R_\varepsilon$ as defined in \cref{def:expansion_operator}, we define the largest reachable safe set starting from the safe seed $\bm x_0$ as the set $S_\varepsilon(\bm x_0)$, obtained by recursively applying $R_\varepsilon$ to $\bm x_0$ an infinite number of times:
\begin{equation}\label{eq:s_epsilon_x0_def}
S_\varepsilon(\bm x_0) \coloneqq \lim_{n\to\infty} \underbrace{R_\varepsilon\big(R_\varepsilon \big(\dots \big(R_\varepsilon}_{n}(\bm x_0)\big).
\end{equation}
\end{definition}
The set $S_\varepsilon(\bm x_0)$ as defined in \cref{def:largest_safe_set}, is, therefore, the largest set that all well behaved GPs can classify as safe starting from $\bm x_0$, when the posterior uncertainty over the safe set at each iteration can be at most of $\varepsilon$.

\subsection{Convergence results}\label{subsec:convergence_results}

In \cref{subsec:definition_of_safe_set}, we formally defined the largest reachable safe set $S_\varepsilon(\bm x_0)$. Here we show that the selection criterion \cref{eq:combined_acqusisition_1} leads to the discovery of the optimum within this set, $f_\varepsilon^*$.
Although \cref{alg:algorithm} can be applied both to discrete and continuous domains alike, in the following analysis we will restrict ourselves to a discrete and finite domain: $|\mathcal{X}| = N < \infty$. All proofs for the results presented in this section can be found in \cref{appendix:ise_proofs}.

As a first result, we show that by sampling only according to $\alpha^{\ourmethodexp}$, i.e., $\bm x_{n+1} \in \argmax_{\bm x \in S_n} \alpha^{\ourmethodexp}(\bm x)$, we eventually classify as safe the whole $S_\varepsilon(\bm x_0)$.
To this end, we introduce the set $\bar{S}(s) \coloneqq \left\{\bm x \in \mathcal{X} : s(\bm x) \geq 0\right\}$, which is the true safe set of the constraint $s$. With it, we can present the following result:
\begin{restatable}{theorem}{ThmExplorationConvergence}
\label{thm:exploration_convergence}
Let the domain $\mathcal{X}$ be discrete and of size $D$: $|\mathcal{X}| = D < \infty$. Assume, moreover, that $\bm{x}_{n+1}$ is chosen according to $\bm x_{n+1} \in \argmax_{\bm x \in S_n} \alpha^{\ourmethodexp}(\bm x)$. Then, if we define $N_S$ as the size of the true safe set of $s$, $N_S = |\bar{S}(s)|$, it holds that for all $\varepsilon > 0$ there exists $N_\varepsilon$ such that, with probability of at least $1 - \delta$, $S_n \supseteq S_\varepsilon(\bm x_0)$ for all $n \geq N_S N_\varepsilon$. 

%
\begin{equation}\label{eq:N_epsilon}
N_\varepsilon = \min\left\{N \in \mathbb{N} : \beta_N \eta^{-1}\left(\frac{C\gamma_N}{N}\right) \leq \varepsilon\right\},
\end{equation}
where $\eta(x) \coloneqq \ln(2)\exp\left\{-c_1\frac{M^2}{x}\right\}\left[1 - \sqrt{\frac{\sigma_\nu^2}{2c_1x + \sigma_\nu^2}}\right]$, $\gamma_N$ is the maximum information capacity of the chosen kernel, and $C = \ln(2) / \sigma_\nu^2\ln(1 + \sigma_\nu^{-2})$.
\end{restatable}

\cref{thm:exploration_convergence} shows that $\alpha^{\ourmethodexp}$ is indeed a good choice for an acquisition function that promotes the expansion of the safe set as much as possible. This property can be combined with the results by \citet{wang_max-value_2017}, who showed that sampling according to $\alpha^{\mes}$ in a fixed domain leads to vanishing regret with high probability, meaning that if we first sample according to $\alpha^{\ourmethodexp}$ until $S_{\hat{n}} \supseteq S_\varepsilon(\bm x_0)$, and then, for $n > \hat{n}$, we sample according to $\alpha^{\mes}$ in $S_{\hat{n}}$, the simple regret with respect to $f_\varepsilon^*$ will decay to zero with high probability.

On the other hand, in general it is not necessary to learn about the whole $S_\varepsilon(\bm x_0)$ in order to uncover the location of the safe optimum, and it will also often be the case that the optimum is not close to the boundary of $S_\varepsilon(\bm x_0)$, so that, in principle, it is not necessary to learn about the whole reachable safe set in order to find it. In such situations, it could be a waste of resources to first learn about the whole $S_\varepsilon(\bm x_0)$ and only then turning the attention to $f_\varepsilon^*$. Rather, as proposed in \cref{eq:combined_acqusisition_1}, one can optimize both acquisition components at the same time. \cref{thm:combined_convergence} offers some convergence guarantees for this scenario, for a particular approximation of \cref{eq:mes_acquisition_expression}. 

Before we can prove \cref{thm:combined_convergence}, however, we first need the following Lemma, which presents a minor twist to a result derived by \citet{wang_max-value_2017} and suggests an equivalent acquisition function to $\alpha^{\mes}$, when we only use one sample to compute the average in \cref{eq:mes_acquisition_expression}.

\begin{restatable}{lemma}{LemmaMesEquivalentAcquisition}
\label{lemma:mes_equivalent_acquisition}
Let $\alpha^{\mes}_{y^*}$ be the mutual information \cref{eq:mes_acquisition_expression} where the average is computed via Monte Carlo with only one sample of $f_{S_n}^*$ of value equal to $y^*$. Then maximizing $\alpha_{y^*}^{\mes}$ is equivalent to maximizing $\hat{\alpha}_{y^*}^{\mes} \coloneqq \sigma_n^2(\bm x) \left(y^* - \mu_n(\bm x)\right)^{-2}$, in the sense that $\argmax_{\bm x} \alpha_{y^*}^{\mes}(\bm x) = \argmax_{\bm x} \hat{\alpha}_{y^*}^{\mes}(\bm x)$.
\end{restatable}

With the equivalent approximation $\hat{\alpha}_{y^*}^{\mes}$ to the $\alpha^{\mes}$ acquisition function, we can now show the following result:

\begin{restatable}{theorem}{ThmCombinedConvergence}
\label{thm:combined_convergence}
Let the domain $\mathcal{X}$ be discrete and of dimension $D$: $|\mathcal{X}| = D < \infty$. Assume, moreover, that $\bm{x}_{n+1}$ is chosen according to $\bm x_{n+1} \in \argmax~\max\left\{\alpha^{\ourmethodexp}(\bm x), \hat{\alpha}_\phi^{\mes}(\bm x)\right\}$, with $\phi \in \mathbb{R}$ such that $\phi > \mu_n(\bm x)$ for all $\bm x \in S_n$ and for all $n$. Then, if we define $N_S$ as the size of the true safe set of $f$, $N_S = |\bar{S}(s)|$, and as $\bm x_\varepsilon^*$ the safe $\argmax$, it holds that for all $\varepsilon > 0$ there exists $N_\varepsilon$ such that, with probability of at least $1 - \delta$, $\bm x_\varepsilon^* \in S_n$ for all $n \geq N_S N_\varepsilon$, and $\sigma_n(\bm x_\varepsilon^*) \leq \varepsilon$ for all $n \geq (N_S + 1) N_\varepsilon$. The smallest of such $N_\varepsilon$ is given by
\begin{equation}\label{eq:N_epsilon_combined}
N_\varepsilon = \min\left\{N \in \mathbb{N} : \beta_N b^{-1}\left(\frac{C\gamma_N}{N}\right) \leq \varepsilon\right\},
\end{equation}
where $b(x) \coloneqq \min\left\{\eta(x), \frac{x}{\phi}\right\}$, with $\eta$ as given in \cref{thm:exploration_convergence}, $\gamma_N = \max_{D \subset \mathcal{X}; |D| = N}I\left(\bm f(D); \bm y(D)\right)$ is the maximum information capacity of the chosen kernel, and $C = \max\left\{\frac{\ln(2)}{\sigma_\nu^2}, \frac{1}{\phi - 2\beta}\right\}$.
\end{restatable}

\cref{thm:combined_convergence} extends the result of \cref{thm:exploration_convergence}, in that it shows that, if we evaluate parameters according to the combined acquisition function \cref{eq:combined_acqusisition_1}, then not only we will learn about $S_\varepsilon(\bm x_0)$, but we will also reduce the uncertainty about the optimum within it, hence solving the safe optimization task.

Both \cref{thm:exploration_convergence,thm:combined_convergence} find convergence bounds on $n$, \cref{eq:N_epsilon,eq:N_epsilon_combined}, that depend on the behavior of $\gamma_n$. In \cref{appendix:ise_proofs}, we show that, in the discrete domain setting, \cref{eq:N_epsilon,eq:N_epsilon_combined} have a non trivial solution.

\section{Discussion and Limitations}\label{sec:discussion}

Since the \ourmethod objective is defined over a continuous domain, our proposed algorithm can deal with higher dimensional domains better than algorithms that rely on a discrete parameter space. However, to find $\bm x_{n+1}$ in \cref{eq:combined_acqusisition_1} we have to solve a non-convex optimization problem with twice the dimension of the the parameter space, which can become a computationally challenging problem as the dimension of the domain grows. 
In higher-dimensional settings, we follow \linebo by \citet{kirschner_adaptive_2019}, which at each iteration selects a random one-dimensional subspace to which it restricts the optimization of the acquisition function.

The selection criterion \cref{eq:combined_acqusisition_1} trades off the expansion of the safe set with the search for the safe optimum. This comparison works best when both the objective function $f$ and the constraint $s$ share the same scale.

In \cref{sec:ise_problem_statement}, we assumed the observation process to be homoskedastic. However, the results can be extended to the case of heteroskedastic Gaussian noise. The observation noise at a parameter $\bm x$ explicitly appears in the \ourmethodexp acquisition function, since it crucially affects the amount of information that we can gain by evaluating the constraint $s$ at $\bm x$. On the contrary, \safeopt-like methods do not consider the observation noise in their acquisition functions. As a consequence, \ourmethod can perform significantly better in an heteroskedastic setting, as we also show in \cref{sec:experiments}.

Lastly, we reiterate that the theoretical safety guarantees offered by \ourmethod are derived under the assumption that $h$ is a bounded norm element of the RKHS space associated with the GP's kernel. In applications, therefore, the choice of the kernel function becomes even more crucial when safety is an issue. For details on how to construct and choose kernels see \citep{garnett_bo_book}. The safety guarantees also depend on the choice of $\beta_n$. Typical expressions for $\beta_n$ include the RKHS norm of the composite function $h$ \citep{kernelized_bandits,fiedler_beta_bounds}, as shown in \cref{eq:beta_def}, which is in general difficult to estimate in practice. Because of this, usually in practice a constant value of $\beta_n$ is used instead.

\section{Experiments}\label{sec:experiments}

In this section, we present experiments where we evaluate the proposed \ourmethod acquisition function on different tasks and investigate its practical performance against competitor algorithms.

\paragraph{GP Samples} For the first part of the experiments, we evaluate \ourmethod on objective and constraint functions $f$ and $s$ that we obtain by sampling a GP prior at a finite number of points. This allows for an in-model comparison of \ourmethod, where we compare its performance to that of a slightly modified version of \safeopt \citep{sui_safe_2015}. In particular, we modify the version of \safeopt that we use in the experiment by defining the safe set in the same way \ourmethod does, i.e., by means of the GP posterior, as done, for example, also by \citet{safe_quadrotors_felix}. 
We also compare against two versions of the \mes algorithm: the standard unconstrained \mes, which we allow to sample also parameters outside of the safe set, and a constrained version of it, \mes-safe, that restricts the optimization of the acquisition function within the safe set.
As metric we use the simple regret, defined as
\begin{equation}\label{eq:simple_regret_def}
r_n \coloneqq \min_{\hat{n} \leq n} f^*_\varepsilon - f(\bm x_n),
\end{equation}
which is the distance between the safe optimum as defined in \cref{def:safe_optimum} and the best safe parameter evaluated so far.

We select 50 independent samples of objective and constraint functions from a two-dimensional GP with RBF kernel, defined in $[-1, 1]\times[-1, 1]$ and run the different algorithms for 100 iterations for each sample. 
As a first experiment, we test the setting where the safety constraint and the objective function are the same, that is $f = s$. 
\cref{fig:2d_gp_samples_same} shows the average simple regret over the 50 samples and the associated standard error over the regret for this setting. From this plot, we can see that \ourmethod achieves comparable results with \safeopt and how the unconstrained version of \mes finds the global optimum, which, in general, is greater than $f^*_\varepsilon$, leading to negative regret in the plots. We also notice that the constrained version of \mes achieves a good performance, comparable with \ourmethod. This result is due to the fact that, when constraint and objective are the same randomly selected GP sample, more often than not the boundaries of the safe set happen to be promising areas also for the location of the current safe optimum, so that also a constrained version of \mes does promote expansion of the safe set, resulting in a good performance.

Subsequently, we move on to explore the setting where objective and constraint are two independently selected GP samples, and plot the corresponding results in \cref{fig:2d_gp_samples_separate}. The results are very similar to the ones shown in \cref{fig:2d_gp_samples_same}, with the only noteworthy difference being the behavior of the constrained version of \mes: while in the case of $f = s$ \mes-safe performs on par with \ourmethod and \safeopt, for $f \neq s$ it achieves very poor results. Such bad performance is a consequence of the safe exploration limits of pure \mes when constrained to a fixed domain, as exemplified in \cref{fig:toy_example}. In the setting $f \neq s$ this limits become relevant, since the boundary of the safe set is in general no longer an interesting location to explore for what concerns the optimum of the objective function.

Finally, in \cref{table:safety_violations} we report the average percentage of safety violations for the various methods, and we notice that the values are comparable for all constrained methods, as the safe set is defined in the same way for all of them.

\begin{figure}[t] 
    \centering
    \subfloat[Same function for constraint and objective]{%
        \includegraphics[width=0.47\textwidth]{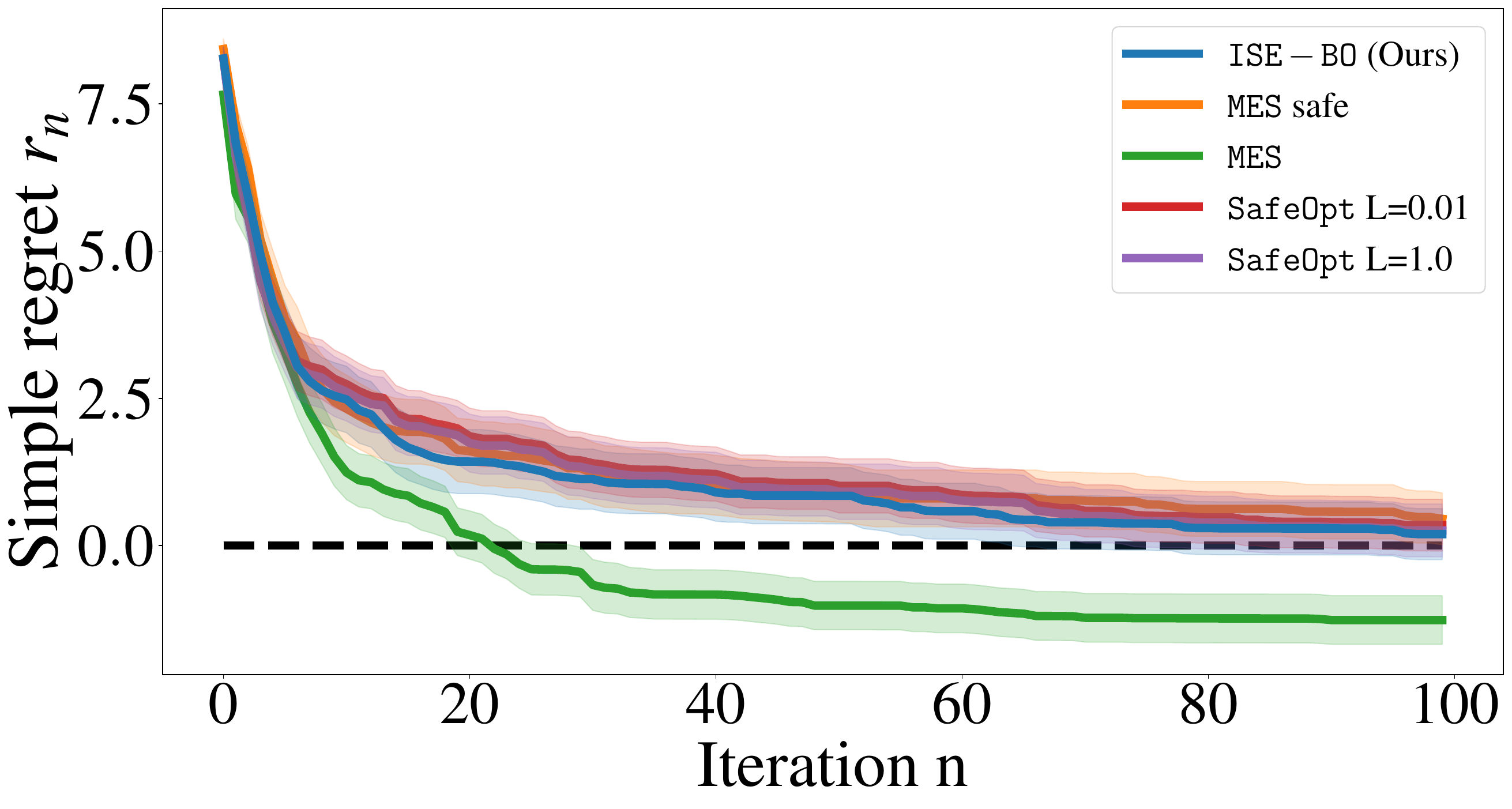}%
        \label{fig:2d_gp_samples_same}%
        }%
    \hfill
    \subfloat[Independent objective and constraint]{%
        \includegraphics[width=0.47\textwidth]{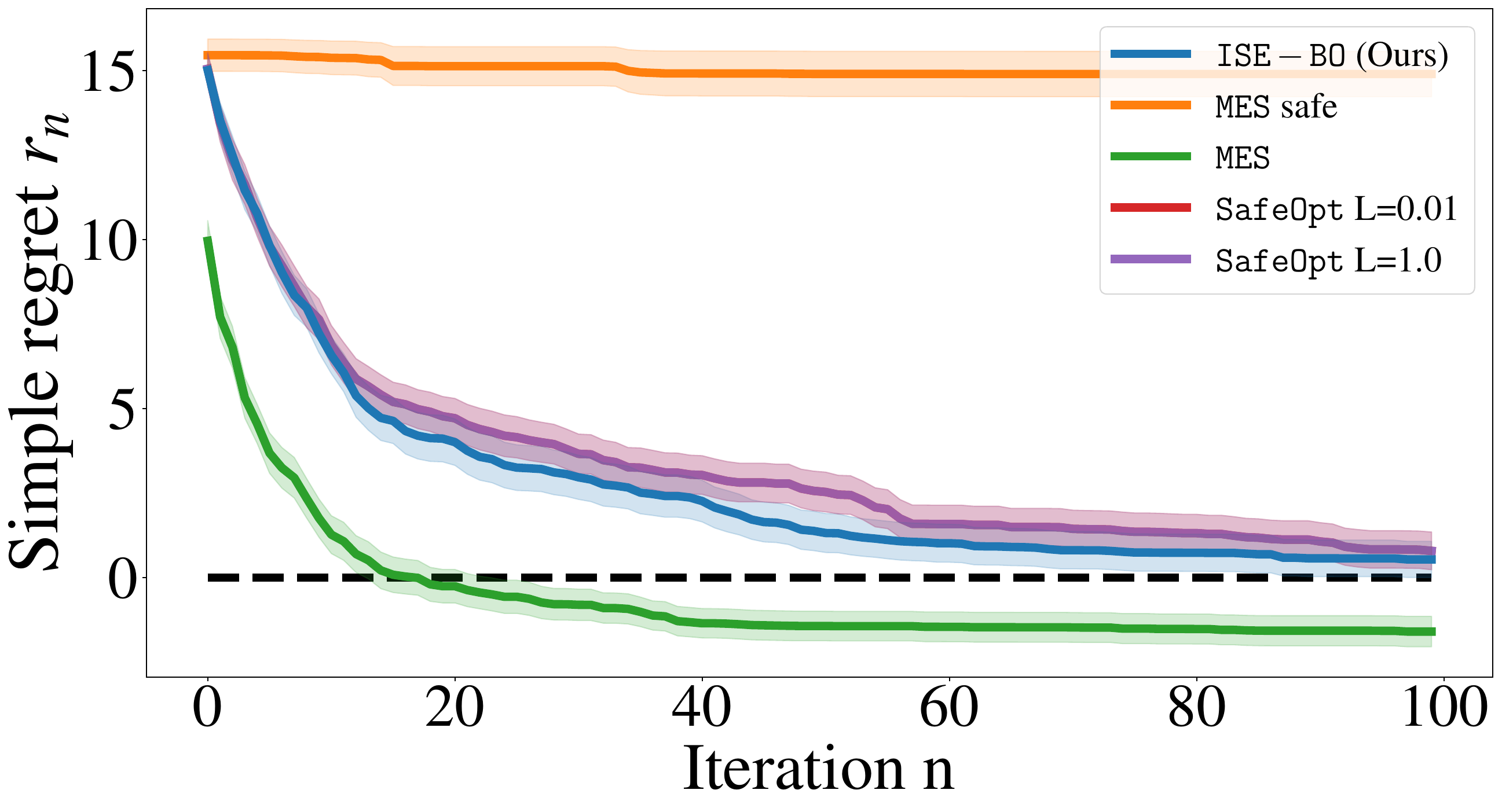}%
        \label{fig:2d_gp_samples_separate}%
        }%
    \caption{Performance of \ourmethod compared to \safeopt and both constrained and unconstrained versions of \mes, when the objective and constraint function are samples drawn the a GP. In both (\subref{fig:2d_gp_samples_separate}) and (\subref{fig:2d_gp_samples_same}) we can see the average simple regret \cref{eq:simple_regret_def} and its standard error for the tested methods over  50 random seeds. In (\subref{fig:2d_gp_samples_same}), the same GP sample served as both the objective $f$ and the constraint $s$, while in (\subref{fig:2d_gp_samples_separate}), for each random seed, two different GP samples were drawn to serve, respectively, as the objective and the constraint.}
    \label{fig:gp_samples_exp}
    \hfill
\end{figure}

\begin{table}[hb]
\adjustbox{max width=\textwidth}{%
\centering
\begin{tabular}{c c c c c c}
\specialrule{.1em}{.05em}{.05em}
\rule{0mm}{4mm}
       & \ourmethod & \mes safe & \mes & \safeopt L = 0.01 & \safeopt L = 1.0\\ 
      \hline
      \% Safety violations & 0.001 $\pm$ 0.002 & 0.004 $\pm$ 0.007 & 0.379 $\pm$ 0.239 & 0.001 $\pm$ 0.001 & 0.001 $\pm$ 0.001 \\
      \specialrule{.1em}{.05em}{.05em}
\end{tabular}}
\caption{Average safety violations per run over the 50 runs used to obtain \cref{fig:gp_samples_exp}}
\label{table:safety_violations}
\end{table}

\paragraph{Synthetic one-dimensional objective} To showcase the advantage of \ourmethod with respect to \safeopt-like approaches, we tested \ourmethod, \mes, \mes-safe and \safeopt for different values of the Lipschitz constant on the synthetic function shown in \cref{fig:one_d_exp_objective}. For simplicity, this function serves both as objective and as constraint.
Starting from the safe seed $\bm x_0 = 0$, in order to find the safe optimum -- marked by the pink square -- the safe optimization algorithm needs to learn up to a very small error the value of the function in the region $[0, 2.3]$, where it gets very close to the safety threshold.

Once it has explored ad learned about the optimum in the region $[-2.5, 0]$, \ourmethod behaves exactly this way, driven by the exploration component $\alpha^\ourmethodexp$ of the acquisition function. On the other hand, \safeopt spends time exploring that region only for appropriate values of the Lipschitz constant, as \cref{fig:1d_experiment_regret} shows. And, naturally, \mes-safe will focus on learning about the local optimum at $\bm x = -2.5$, away from the interesting region. 

\begin{figure}[t] 
    \centering
    \subfloat[One-dimensional objective-constraint]{%
        \includegraphics[width=0.47\textwidth]{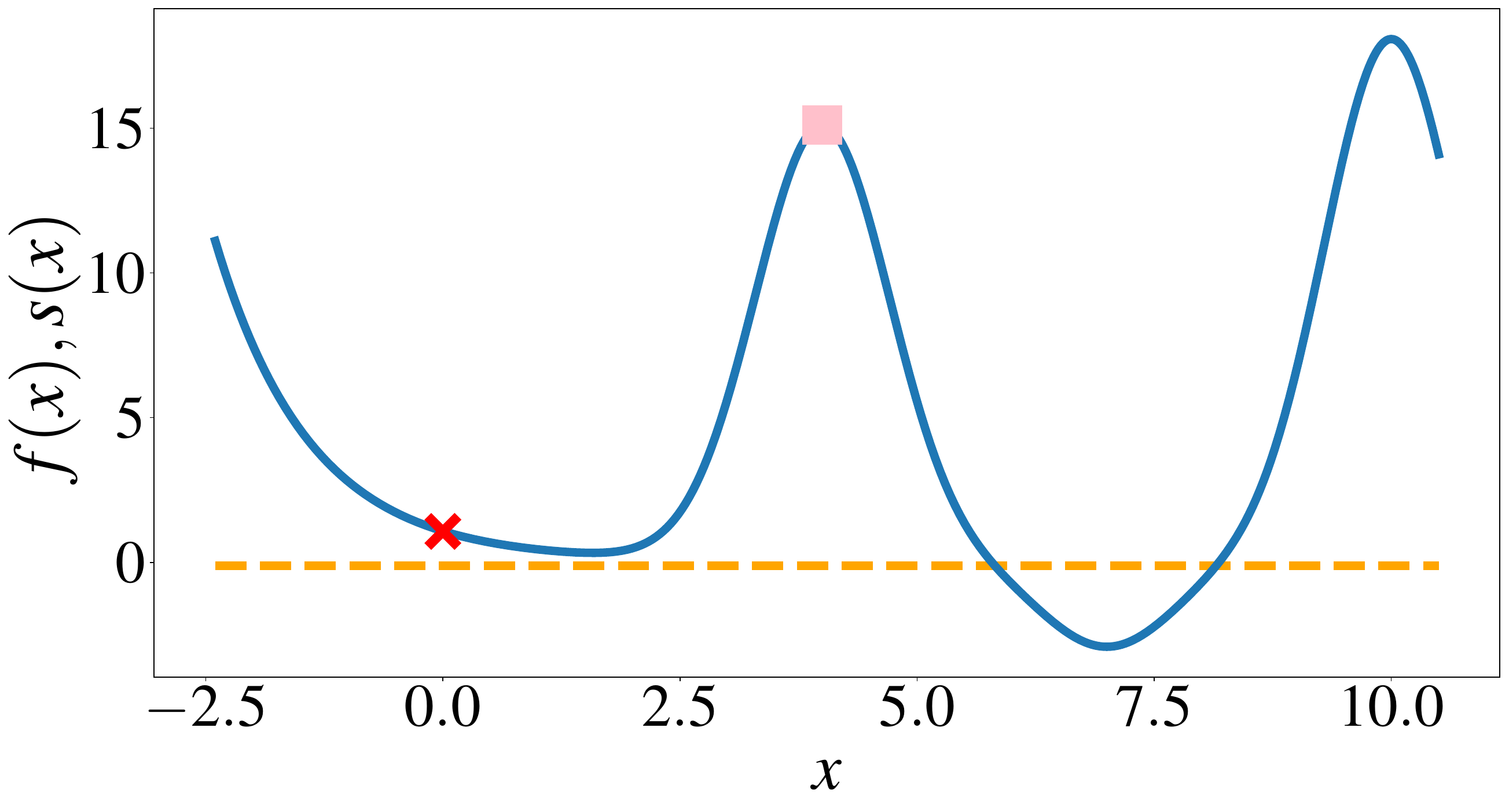}%
        \label{fig:one_d_exp_objective}%
        }%
    \hfill
    \subfloat[Simple regret comparison]{%
        \includegraphics[width=0.47\textwidth]{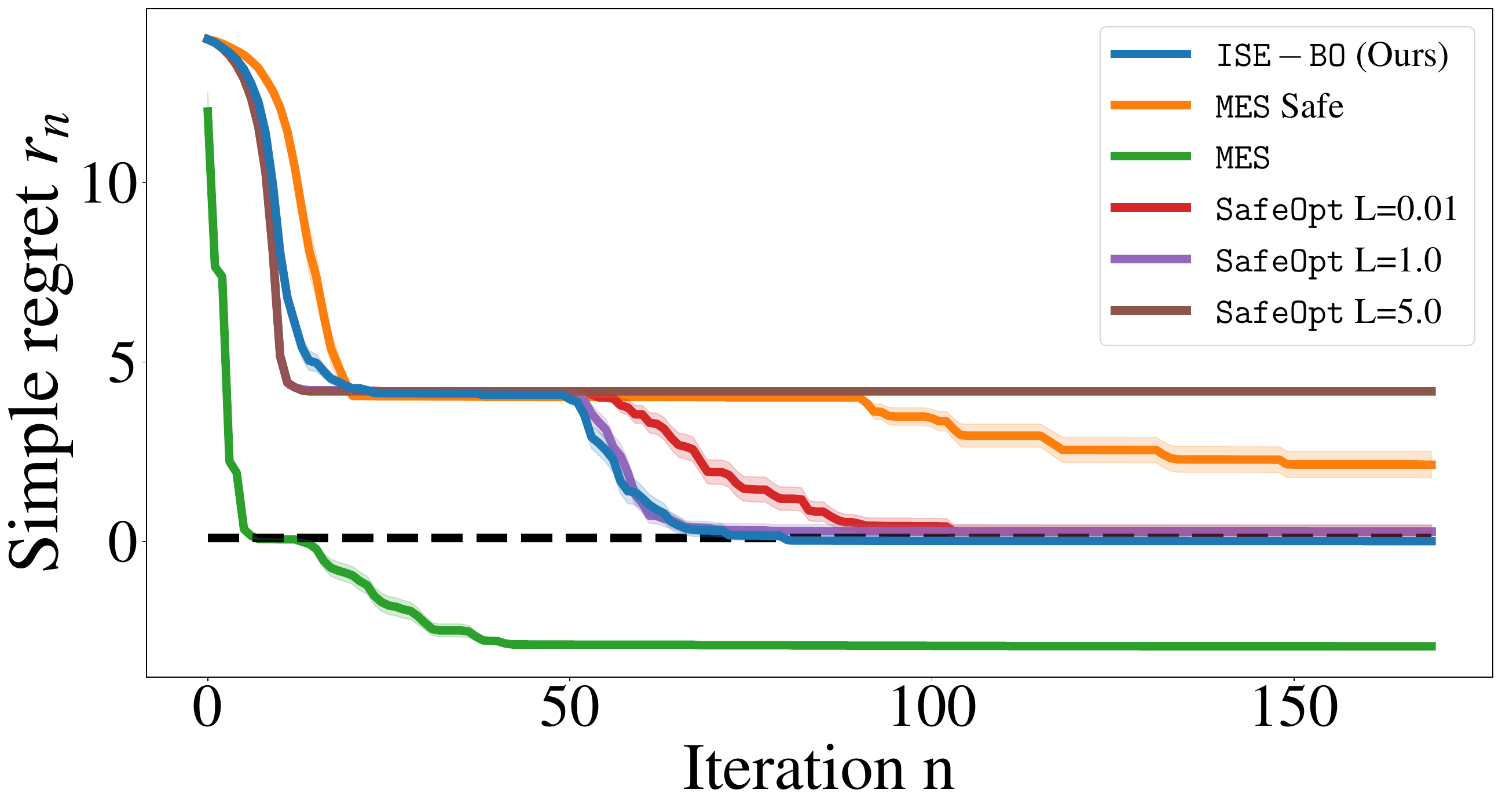}%
        \label{fig:1d_experiment_regret}%
        }%
    \caption{Performance of \ourmethod compared to \safeopt and both constrained and unconstrained versions of \mes for a one dimensional synthetic function, with $f = s$. In (\subref{fig:one_d_exp_objective}), we plot this function, with the value of the safe seed indicated by the red cross and the one of the safe optimum by the pink square, while in (\subref{fig:1d_experiment_regret}) we plot the average simple regret and its standard error over 50 random seeds. \ourmethod will first focus on the left of the safe seed. Once it has learned about the value of the optimum at $\bm x = -2.5$, it will it find more informative to explore to the right, eventually uncovering the region containing $f^*_\varepsilon$. On the other hand, \mes-safe will spend most of its time around the local optimum at $\bm x = -2.5$. Similarly, \safeopt algorithms with a too low or too high value of the Lipschitz constant will take longer to bypass the region where the constraint is very close to the safety threshold.}
    \label{fig:1d_experiments}
    \hfill
\end{figure}

\paragraph{Heteroskedastic noise domains}\label{par:high_heteroskedastic}
As discussed in \cref{sec:discussion}, for high dimensional domains, we can follow a similar approach to \linebo, limiting the optimization of the acquisition function to a randomly selected one-dimensional subspace of the domain. Moreover,  it is also interesting to test \ourmethod in the case of heteroskedastic observation noise, since the noise is a critical quantity for our acquisition function, while it does not affect the selection criterion of \safeopt-like methods. Therefore, in this experiment we combine a high dimensional problem with heteroskedastic noise. In particular, we apply a \linebo version of \ourmethod to the constraint and objective function $f(\bm x) = \frac{1}{2}e^{-\bm x^2} + e^{-(\bm x \pm \bm x_1)^2} + 3e^{-(\bm x \pm \bm x_2)^2} + 0.2$ in dimension four and six, with the safe seed being the origin. This function has two symmetric global optima at $\pm \bm x_2$ and we set two different noise levels in the two symmetric domain halves containing the optima. To assess the exploration performance, also in this case we use the simple regret. As the results in \cref{fig:high_d_line} show, \ourmethod achieve a superior sample efficiency than \safeopt, which, in turn, as expected obtains a greater performance that the constrained version of \mes, for the reasons that we discussed in \cref{sec:safe_optimization_algorithm}. Namely, for a given number of iterations, by explicitly exploiting knowledge about the observation noise, \ourmethod is able to classify as safe regions of the domain further away from the origin, in which the objective function assumes its largest values, resulting in a smaller regret. On the other hand, \safeopt-like methods only focus on the posterior variance, sampling parameters in both regions of the domain with with similar frequency, so that the higher observation noise causes them to remain stuck in a smaller neighborhood of the origin, resulting in bigger regret.

\begin{figure*}
\hfill
     \centering
     \begin{subfigure}{0.495\textwidth}
         \centering
         \includegraphics[width=\textwidth,height=0.16\textheight]{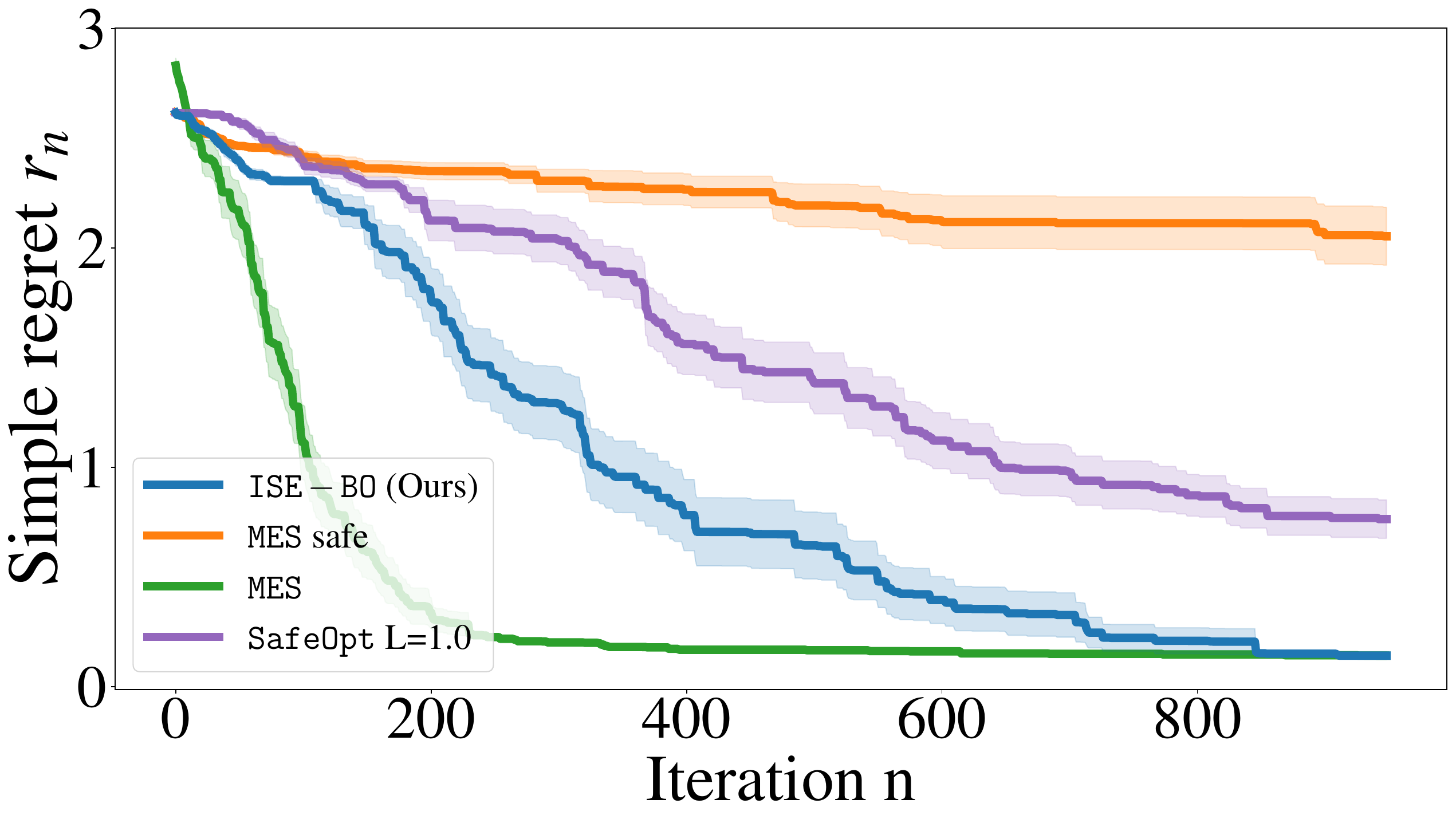}
         \caption{Dimension 4.}
         \label{fig:high_d_line_4}
     \end{subfigure}
     \begin{subfigure}{0.495\textwidth}
         \centering
         \includegraphics[width=\textwidth,height=0.16\textheight]{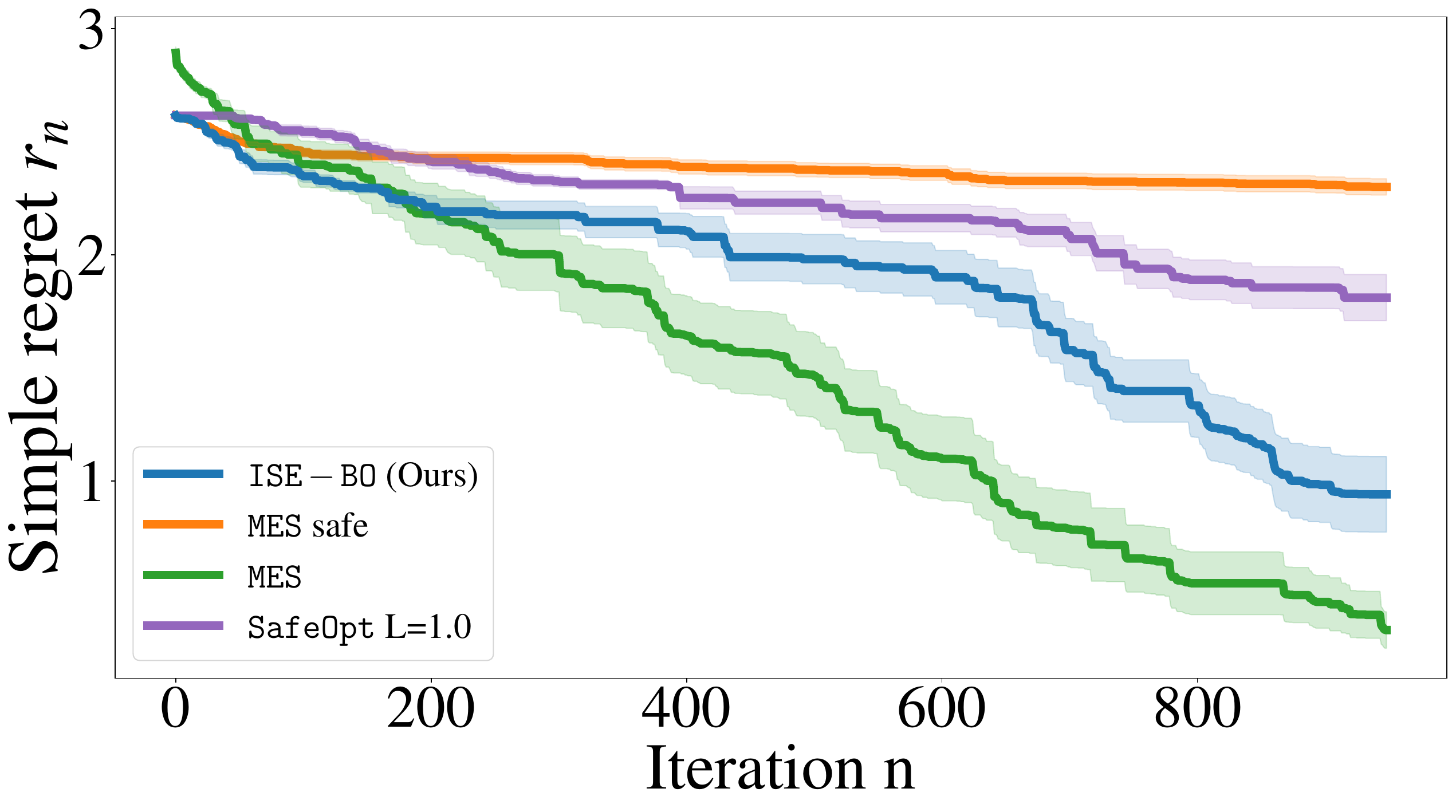}
         \caption{Dimension 6.}
         \label{fig:high_d_line_6}
     \end{subfigure}
        \caption{Performance comparison in higher dimension with heteroskedastic noise, for $d=$4 and $d=$6, where we plot the simple regret $r_n$ with respect to the safe optimum. (\subref{fig:high_d_line_4}) and (\subref{fig:high_d_line_6}) show, respectively, the average regret over 30 runs in dimension four and six, as a function of the number of iterations. We can see that this adapted version of \ourmethod promotes expansion of the safe set, leading to classifying as safe regions where the latent function attains its largest value. The plots also show that \ourmethod achieves a better sample efficiency than both \stageopt-like exploration and the \stageopt inspired heuristic acquisition. As one would expect, we also see that the constrained version of \mes achieves the worst performance, as it lacks motivation to expand the safe set. On the other hand, the unconstrained \mes algorithm is naturally the quickest at identifying the optimum, since it is free to evaluate any parameter in the domain.}
        \label{fig:high_d_line}
\end{figure*}

\paragraph{OpenAI Gym control}
After investigating the performance of \ourmethod on GP samples, we apply its exploration component to a classic control task from the OpenAI Gym framework \citep{brockman2016openai}, with the goal of finding the set of parameters of a controller that satisfy some safety constraint. In particular, we consider a linear controller for the inverted pendulum.
The controller is given by $u_t = x_1 \theta_t + x_2 \dot{\theta}_t$, where $u_t$ is the control signal at time $t$, while $\theta_t$ and $\dot{\theta}_t$ are, respectively, the angular position and the angular velocity of the pendulum. Starting from a position close to the upright equilibrium, the controller's task is the stabilization of the pendulum, subject to a safety constraint on the maximum velocity reached within one episode. For a given initial controller configuration $\bm x_0 \coloneqq (x_1^0, x_2^0)$, we want to explore the controller's parameter space avoiding configurations that lead the pendulum to swing with a too high velocity. Namely, the safety constraint is the maximum angular velocity reached by the pendulum during an episode of fixed length $T$, $s(\bm x) = \max_{t\leq T} \dot{\theta}_t(\bm x)$, and the safety threshold is not at zero, but rather at some finite value $\dot{\theta}_M = 0.5\text{rad/}s$, so that the safe parameters are those for which the maximum velocity is below $\dot{\theta}_M$. In \cref{fig:true_pendulum_safe_set}, we show the true safe set for this problem, while in \cref{fig:pendulum_discover_safe_set_1,fig:pendulum_discover_safe_set_2,fig:pendulum_discover_safe_set_3} one can see how \ourmethodexp safely explores the true safe set. 
These plots show the parameters that \ourmethodexp chooses to evaluate at different iterations. We see that in all these occasions, the \ourmethodexp acquisition function focuses on the boundary of the current safe set, as this is the region that is naturally most informative about the safety of parameters outside of the safe set. This way, \ourmethodexp is able to expand the safe set, until it learns the constraints function with high confidence in the vicinity of the boundary, at which point the safe set can only change slightly. This behavior eventually leads to the full true safe set to be classified as safe by the GP posterior, as \cref{fig:pendulum_discover_safe_set_3} shows.

\begin{figure}[ht]
\hfill
     \centering
     \begin{subfigure}[b]{0.49\textwidth}
         \centering
         \includegraphics[width=\textwidth,height=0.165\textheight]{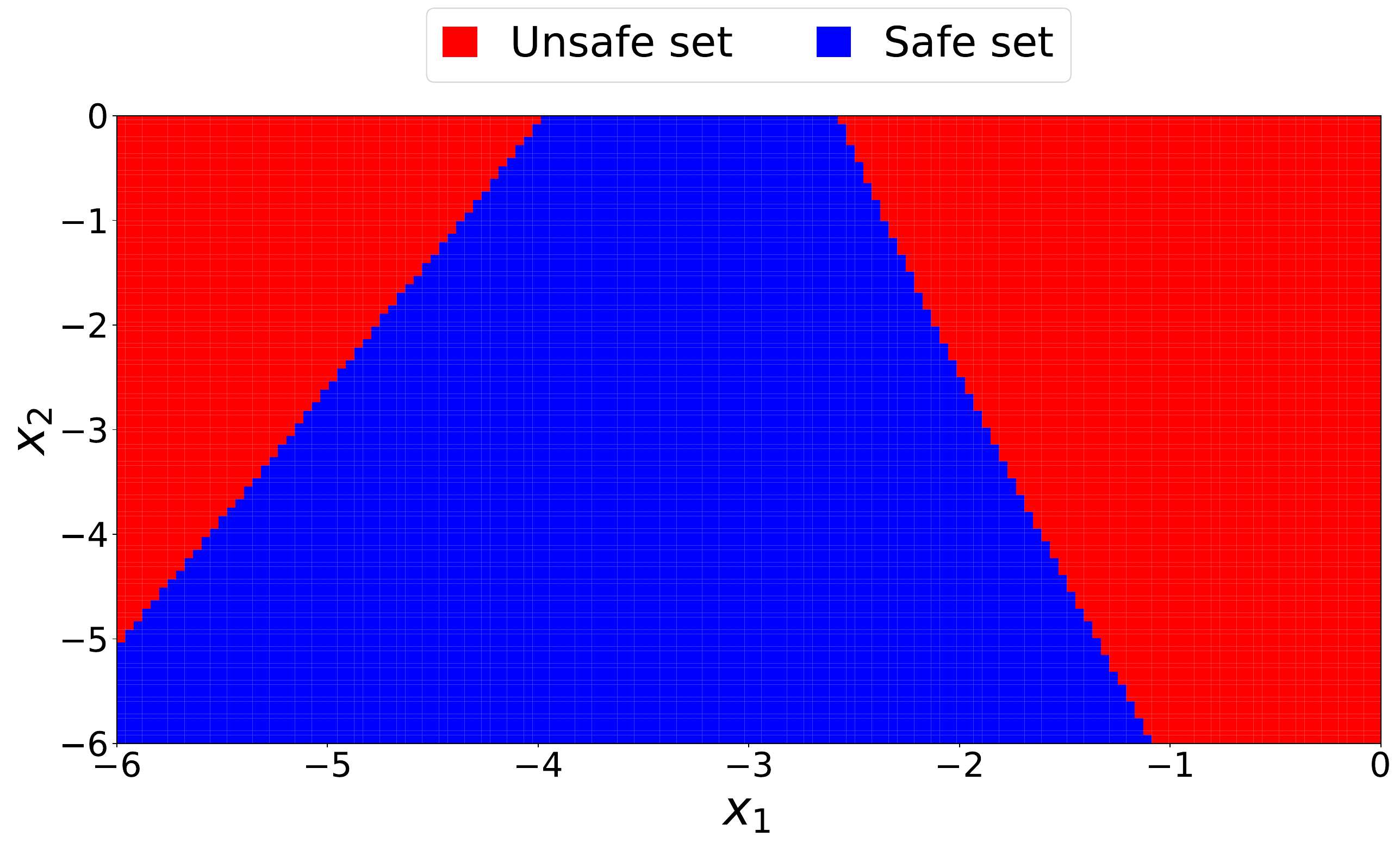}
         \caption{True safe set}
         \label{fig:true_pendulum_safe_set}
     \end{subfigure}
     \hfill
     \begin{subfigure}[b]{0.49\textwidth}
         \centering
         \includegraphics[width=\textwidth,height=0.165\textheight]{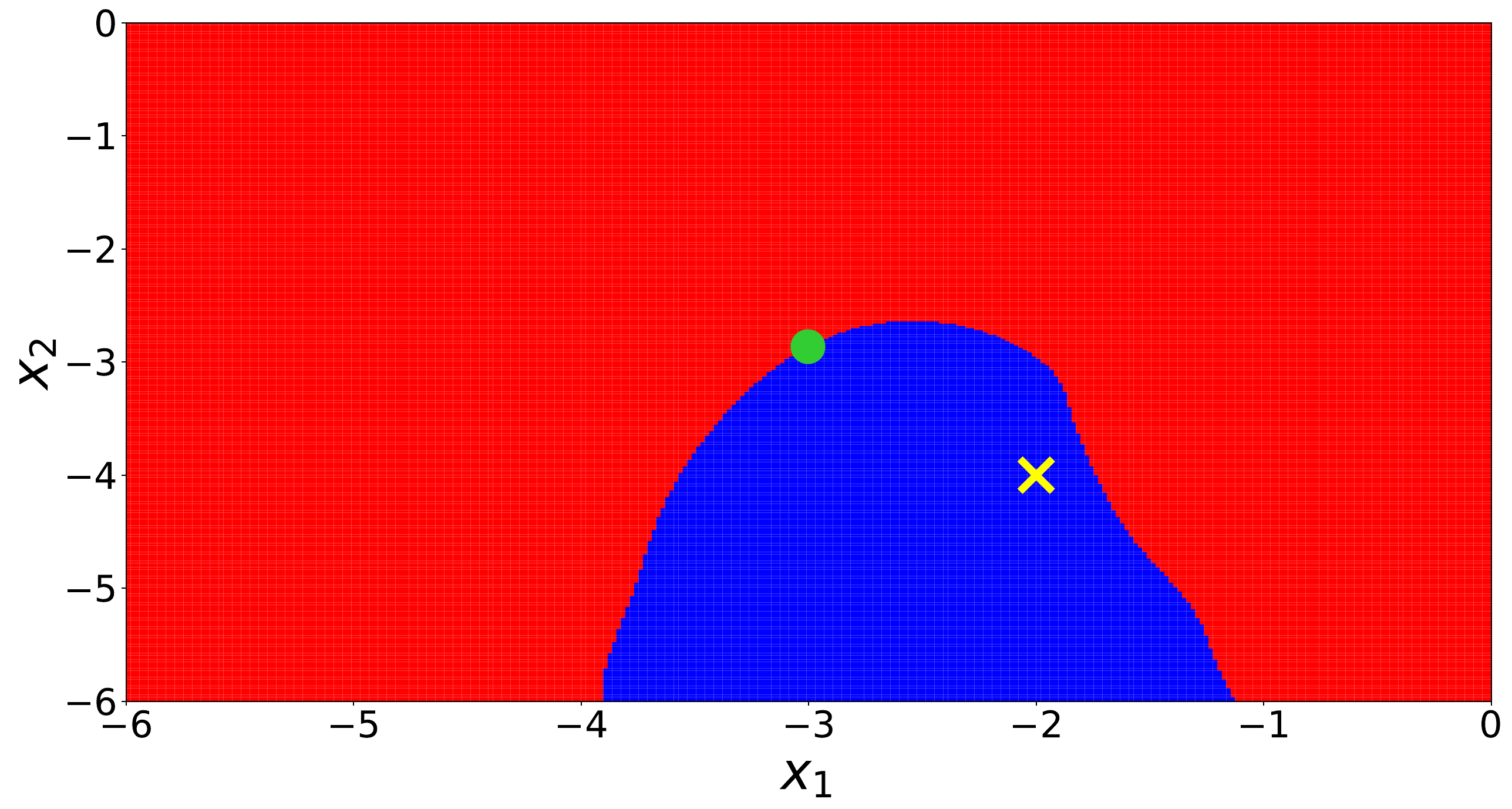}
         \caption{Iteration 10}
         \label{fig:pendulum_discover_safe_set_1}
     \end{subfigure}
     \hfill
     \begin{subfigure}[b]{0.49\textwidth}
         \centering
         \includegraphics[width=\textwidth,height=0.165\textheight]{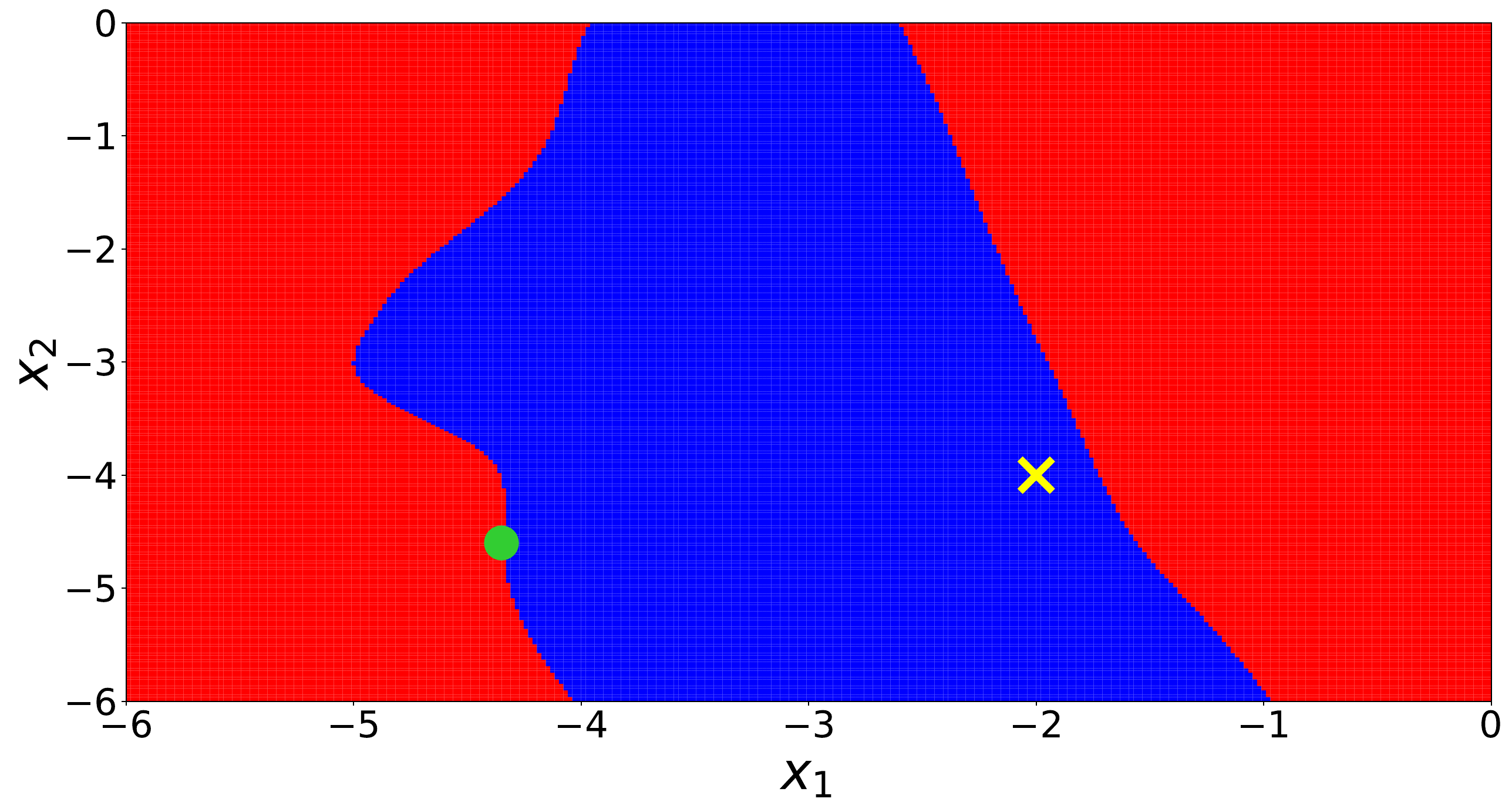}
         \caption{Iteration 30}
         \label{fig:pendulum_discover_safe_set_2}
     \end{subfigure}
     \hfill
     \begin{subfigure}[b]{0.49\textwidth}
         \centering
         \includegraphics[width=\textwidth,height=0.165\textheight]{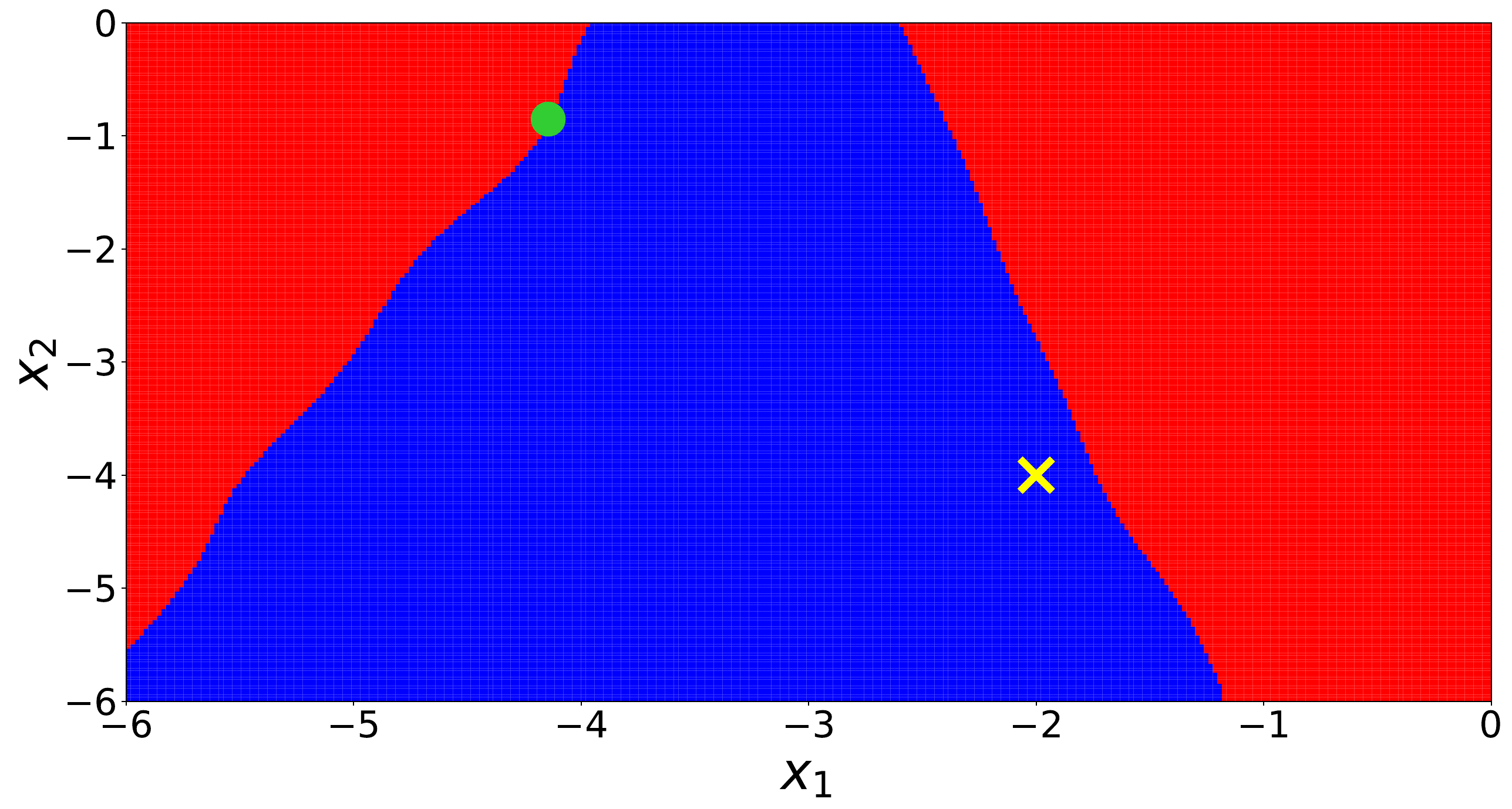}
         \caption{Iteration 50}
         \label{fig:pendulum_discover_safe_set_3}
     \end{subfigure}
        \caption{Safe exploration of the linear controller's parameter space in the inverted pendulum experiment. In (\subref{fig:true_pendulum_safe_set}) we see the true safe set, while in (\subref{fig:pendulum_discover_safe_set_1}-\subref{fig:pendulum_discover_safe_set_3}) we see the safe set (blue region) as learned by \ourmethodexp at various iterations. The point marked by the green dot is $\bm x_{n+1} = \argmax \alpha^\ourmethodexp$, while the yellow cross is the initial safe seed $\bm x_0$.}
        \label{fig:pendulum_experiment}
\end{figure}

\paragraph{Rotary inverted pendulum controller}
Next, we consider a simulated Furuta pendulum \citep{furuta_pendulum} and an associated Proportional-Derivative (PD) controller, for which we want to find the optimal parameters. In particular, we consider an initial state close to the upward equilibrium state and look for parameters $\bm x^*$ that maintain that equilibrium with the least effort possible. 
The state of the pendulum is given by the four dimensional vector $[\theta, \alpha, \dot{\theta}, \dot{\alpha}]$, where $\theta$ is the horizontal angle, $\alpha$ the vertical one, and $\dot{\theta}, \dot{\alpha}$ their respective angular velocities. The objective function is given by the negative sum of the velocities and actions collected along an episode: $f(\bm x) \coloneqq -\sum_t\left(|\dot{\theta}_t| + |\dot{\alpha}_t| + |a_t|\right)$, where $\bm x$ are the controller parameters, which the values of the angles and angular velocities during an episode depend upon, while $a_t$ is the action that the controller outputs at time-step $t$. To maximize this objective while maintaining the pendulum close to its upward equilibrium means to find a controller that is able to keep the pendulum up with the minimum amount of movement and effort. Closeness to the upward equilibrium position is encoded in the safety constraint: $s(\bm x) \coloneqq -\sum_t |\alpha_t| + b$, where $b$ is an offset term to ensure that the safety threshold is at $s(\bm x) = 0$. A big value of $s$ means that $\alpha$ has deviated only slightly from the upward equilibrium point $\alpha = 0$ during the episode, while bigger values of $s$ mean larger deviation. The PD controller has one proportional parameter and one derivative parameter associated with $\theta$ and another pair associated to $\alpha$: $\bm x \coloneqq [k_p^\theta, k_d^\theta, k_p^\alpha, k_d^\alpha]$. In this experiment, we fix $k_p^\theta$ and $k_p^\alpha$, and consider the problem of finding the optimal safe $k_d^\theta$ and $k_d^\alpha$.
\cref{fig:pd_ctrl_exp} summarizes the results of the experiment. In this plot, we can see that also in this case \ourmethod and \safeopt achieve a comparable result, eventually learning about the safe optimum, while the safely constrained version of \mes struggles to efficiently explore and uncover the region of the domain that contains the safe optimum, leading to a slower convergence compared to \ourmethod. On the other hand, as expected we can see that the unconstrained \mes algorithm quickly converges to the safe optimum (which in this case is also the global one) since it is able to freely explore any region of the domain.

\begin{figure}[t] 
  \centering
  \includegraphics[width=0.47\textwidth]{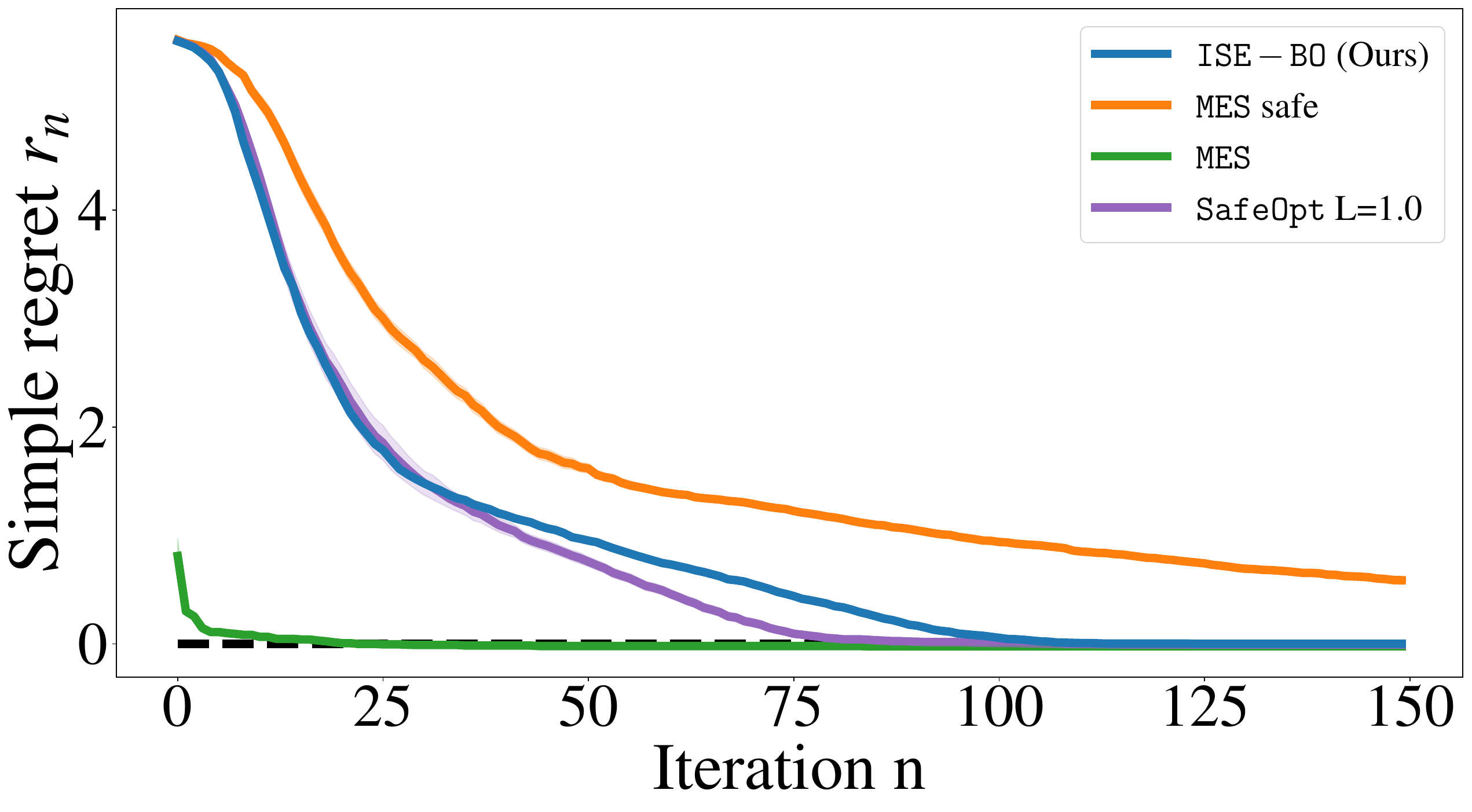}
  \caption{Performance of \ourmethod compared to \safeopt and both constrained and unconstrained versions of \mes for the Furuta pendulum PD controller experiment. We plot the average over 30 random seeds and the corresponding standard error. The plot shows that, also in this case, \ourmethod performs comparably with \safeopt and that both outperform the safety constrained version of \mes. In this experiment, the safe optimum is also the global optimum, which explains why the green line, corresponding unconstrained version of \mes, quickly converges to zero regret.}
  \label{fig:pd_ctrl_exp}
\end{figure}

\section{Conclusions}

We have introduced \ourmethodlong (\ourmethod), a safe BO algorithm that pairs Max-Value Entropy Search with a novel approach to safely explore a space where the safety constraint is \textit{a priori} unknown. \ourmethod efficiently and safely explores by evaluating only parameters that are safe with high probability and by choosing those parameters that yield the greatest information gain about either the safety of other parameters or the value of the safe optimum. We theoretically analyzed \ourmethod and showed that it leads to arbitrary reduction of the uncertainty in the largest reachable safe set containing the starting parameter. Our experiments support these theoretical results and demonstrate an increased sample efficiency of \ourmethod compared to \safeopt-based approaches.

\vskip 0.2in
\bibliography{references}

\appendix
\section{Proofs - Not polished}\label{appendix:ise_proofs}

Let's start by proving \cref{lemma:safe_set_is_safe_with_high_probability,lemma:in_beta_gp_with_high_prob}, while we present the proof of \cref{lemma:mutual_infos_decrease_with_sigma} later, as we need some additional results to be able to prove it. In the following, when it is clear from the context, for the sake of simplicity we drop the superscripts $f$ and $s$ in the notation of the posterior mean and variance.

\SafeSetHighProbab*
\begin{proof}
Under the assumptions of the Lemma, \citet{kernelized_bandits} have shown that with the choice of $\beta_n = B + R\sqrt{2\left(\ln(e/\delta) + \gamma_n\right)}$, it holds with probability of at least $1 - \delta$ that $s(\bm x) \in \left[\mu_n^s(\bm x) \pm \beta_n\sigma_n^s(\bm x)\right]$ for all $\bm x \in \mathcal{X}$ and for all $n$, with $\gamma_n$ being the maximum information gain about $h$ after $n$ iterations. This result, together with the way we defined the safe set in \cref{eq:safe_set_definition}, implies the claim of the Lemma, concluding the proof. 
\end{proof}

\begin{lemma}
\label{lemma:in_beta_gp_with_high_prob}
Let $s, f$ and $h$ be as in \cref{eq:h_definition}, and let $h$ have bounded RKHS norm, as in \cref{lemma:safe_set_is_safe_with_high_probability}. 
Let, moreover, $\mathcal{D}$ be a set of observations $\mathcal{D} = \left\{(\bm x_n, y_n)\right\}_{n=1}^{|\mathcal{D}|}$. Then for all $\delta \in (0, 1)$, with probability of at least $1 - \delta$ it holds that $s(\bm x) \in \left[\mu_n^s(\bm x) \pm \beta_n\sigma_n^s(\bm x)\right]$ for all $\bm x \in \mathcal{X}$ and for all $n \leq |\mathcal{D}|$, with $\beta_n$ as defined in \cref{eq:beta_def}.
\end{lemma}

\begin{proof}
As we recalled in the proof of \cref{lemma:safe_set_is_safe_with_high_probability}, \citet{kernelized_bandits} have shown that with a choice of $\beta_n = B + R\sqrt{2\left(\ln(e/\delta) + \gamma_n\right)}$, it holds with probability of at least $1 - \delta$ that $s(\bm x) \in \left[\mu_n^s(\bm x) \pm \beta_n\sigma_n^s(\bm x)\right]$ for all $\bm x \in \mathcal{X}$ and for all $n$. If we now take a GP prior and a dataset $\mathcal{D}$ such that $\mu_n$ and $\sigma_n$ define the posterior $GP$ obtained by conditioning on the subset of $\mathcal{D}$ $\mathcal{D}_n = \{(\bm x_i, y_i)\}_{i = 1}^n$, then the result by \citet{kernelized_bandits} means that, with probability of at least $1 - \delta$,  $s(\bm x) \in \left[\mu_n^s(\bm x) \pm \beta_n\sigma_n^s(\bm x)\right]$ for all $\bm x$ and for all $n \leq \abs{\mathcal{D}}$.
\end{proof}

Now, in order to prove \cref{thm:exploration_convergence,thm:combined_convergence} we will first show that the result holds for any acquisition function that satisfies some specific properties, and then show that our acquisitions satisfy those properties.

\begin{lemma}\label{lemma:generic_form_vanishing_variance}
Assume that $\bm{x}_{n+1}$ is chosen according to $\bm x_{n+1} \in \argmax_{\bm x \in \mathcal{X}_n}\alpha_n(\bm x)$, where $\mathcal{X}_n \subseteq \tilde{\mathcal{X}}$ is a subset of the domain of $h$, possibly different at every iteration $n$, but such that $\mathcal{X}_{n+1} \subseteq \mathcal{X}_n$, while $\alpha_n$ is an acquisition function that satisfies the two properties:
\begin{itemize}
    \item $\alpha_n(\bm x_{n+1}) \leq C_0\sigma_n^2(\bm x_{n+1})$ for some constant $C_0 > 0$;
    \item $\alpha_n(\bm x_{n+1}) \geq g(\tilde{\sigma}_n^2)$, with $g$ being a monotonically increasing, and with ${\tilde{\sigma}_n^2 \coloneqq \max_{\bm x \in \mathcal{X}_n}\sigma_n^2(\bm x)}$.
\end{itemize}
Then it holds that for all $\varepsilon > 0$ there exists $N_\varepsilon$ such that $\sigma_n^2(\bm x) \leq \varepsilon$ for every $\bm x \in \mathcal{X}_n$ if $n \geq \hat{n} + N_\varepsilon$. The smallest of such $N_\varepsilon$ is given by
\begin{equation}\label{eq:N_epsilon_generic}
N_\varepsilon = \min\left\{N \in \mathbb{N} : g^{-1}\left(\frac{C_1\gamma_N}{N}\right) \leq \varepsilon\right\},
\end{equation}
where $\gamma_N = \max_{D \subset \mathcal{X}; |D| = N}I\left(\bm f(D); \bm y(D)\right)$ is the maximum information capacity of the chosen kernel \citep{srinivas_gaussian_2010,gp_opt_with_mi}, and $C_1 = C_0 / \sigma_\nu^2\ln(1 + \sigma_\nu^{-2})$.
\end{lemma}

\begin{proof}
Thanks to the assumptions on the aquisition function $\alpha_n$, we know that for every $n$ it holds that
\begin{equation}
\begin{split}
g(\tilde{\sigma}_n^2) \leq \alpha_n(\bm x_{n+1}) \leq C_0\sigma_n^2(\bm x_{n+1}) \\
\Longrightarrow g(\tilde{\sigma}_n^2) \leq C_0\sigma_n^2(\bm x_{n+1})
\end{split}
\end{equation}
we can now exploit the monotonicity of $g$ and the fact that $\tilde{\sigma}_n^2$ is not increasing if the set $\mathcal{X}_n$ does not expand with $n$:
\begin{equation}
g(\tilde{\sigma}_n^2) \leq g(\tilde{\sigma}_m^2) \leq C_0\sigma_m^2(x_{m+1}) ~~~~~~ \forall n \geq m 
\end{equation}
we can then use this result to write:
\begin{equation}\label{eq:g_bounded_by_gamma}
\begin{split}
ng(\tilde{\sigma}_n^2) = \sum_{i = 1}^{n}g(\tilde{\sigma}_n^2) &\leq \\
C_0\sigma_\nu^2\sum_{i = 1}^{n}\sigma_\nu^{-2}\sigma_i^2(\bm x_{i+1}) &\leq \\
\frac{C_0\sigma_\nu^2}{\sigma_\nu^2\ln(1 + \sigma_\nu^{-2})}\sum_{i = 1}^{n}\ln\left(1 + \sigma_\nu^{-2}\sigma_i^2(\bm x_{i+1})\right) &\leq \\
C_1\gamma_n
\end{split}
\end{equation}
where $\gamma_n$ is the maximum information capacity and $C_1 = \frac{C_0\sigma_\nu^2}{\sigma_\nu^2\ln(1 + \sigma_\nu^{-2})}$. The second last passage follows from the fact that $x \leq \ln(1 + x) \sigma_\nu^{-2} / \ln(1 + \sigma_\nu^2)$ for $x \in [0, \sigma_\nu^{-2}]$ together with the fact that $\sigma_\nu^{-2}\sigma_i^2(\bm x_{i+1}) \leq \sigma_\nu^{-2}k(\bm x_{i+1}, \bm x_{i+1}) \leq \sigma_\nu^{-2}$. Finally, the last passage uses the representation of the mutual information $I(\{y_n\}; \{f(\bm x_n)\})$ presented by \citet{srinivas_gaussian_2010}.

Using \cref{eq:g_bounded_by_gamma}, we can show that the minimum $N_\varepsilon$ satisfying the claim of the theorem is the one given by \cref{eq:N_epsilon_generic}:
\begin{equation}\label{eq:N_epsilon_generic_proof}
N_\varepsilon = \min\left\{N \in \mathbb{N} : g^{-1}\left(\frac{C_1\gamma_N}{N}\right) \leq \varepsilon\right\}
\end{equation}
and we are now able to conclude that, as long as the information capacity grows sub-linearly in $N$, the set on the r.h.s. of~\cref{eq:N_epsilon_generic_proof} is not empty $\forall \varepsilon > 0$. This is guaranteed by the fact that $g^{-1}$ is monotonically increasing, since so is its inverse $g$.
To check that this $N_\varepsilon$ indeed satisfies the claim, one just has to apply $g^{-1}$ on both initial and final state of \cref{eq:g_bounded_by_gamma} and then substitute $N_\varepsilon$ in the place of $n$; the rest follows from the fact that the maximum variance is non increasing on $\mathcal{X}_n$ thanks to the assumption that $\mathcal{X}_{n+1} \subseteq \mathcal{X}_{n}$ for all $n$.
\end{proof}

\cref{lemma:generic_form_vanishing_variance} tells us that if an acquisition function satisfies the conditions of the Lemma and it is only allowed to sample within a non-expanding set, then it will cause the posterior variance to vanish over such set. Now we can use this result to show that, if we employ such an acquisition function to sample within $S_n$ at each iteration, then, starting from $\bm x_0$, with high probability we will eventually classify as safe the whole $S_\varepsilon(\bm x_0)$.
We start by deriving a sufficient condition for $S_n$ to include $S_\varepsilon(\bm x_0)$.

\begin{definition}[Per GP expansion operator]
Similarly to \cref{def:expansion_operator}, given a subset $\Omega$ of the domain, for each GP$_\mathcal{D}$ in the set $\bm \beta$-GP$_s^\varepsilon(\Omega)$, we can define the \emph{per GP} expansion operator $R^{\text{GP}_\mathcal{D}}_\varepsilon$ as
\begin{equation}
R^{\text{GP}_\mathcal{D}}_\varepsilon(\Omega) = \{\bm x \in \mathcal{X}: \bm x \in S_{\text{GP}_\mathcal{D}} \}.
\end{equation}
\end{definition}
Given the current safe set $S_n$, the set $R^{\text{GP}_\mathcal{D}}_\varepsilon(S_n)$ would then be the set of all points classified as safe by a specific GP in $\bm \beta$-GP$_s^\varepsilon(S_n)$, i.e.\@ by a particular well behaved GP whose uncertainty is smaller than $\varepsilon$ on $S_n$. Once we recall the definition of the expansion operator $R_\varepsilon$ \cref{eq:expansion_operator}, we see that $R_\varepsilon(S_n)$ is nothing else than the intersection of all $R^{\text{GP}_\mathcal{D}}_\varepsilon(S_n)$ for all $\text{GP}_\mathcal{D} \in \bm \beta$-GP$_s^\varepsilon(S_n)$, which offers a proof for the following result:

\begin{lemma}\label{lemma:inclusion_for_per_gp_operator}
Let $\Omega$ be a subset of the domain, then for any GP$_\mathcal{D}$, GP$_\mathcal{D}$ $\in$ $\bm \beta\text{-GP}_f^\varepsilon(\Omega)$ implies the inclusion $R^{\text{GP}_\mathcal{D}}_\varepsilon(\Omega) \supseteq R_\varepsilon(\Omega)$.
\end{lemma}

\begin{proof}
The claim follows directly by the definitions of the two operators $R^{\text{GP}_\mathcal{D}}_\varepsilon$ and $R_\varepsilon$, since a parameter $\bm x$ that is not classified as safe by the specific GP$_\mathcal{D}$ is not, by definition, part of $R_\varepsilon(\Omega)$.
\end{proof}

With the help of the expansion operator $R^{\text{GP}_\mathcal{D}}_\varepsilon$ it is also straightforward to prove the following lemma:

\begin{lemma}\label{lemma:safe_set_at_epsilon_includes_reachable_set}
Fix $\varepsilon > 0$ and let $\hat{n}$ be such that $\beta_{\hat{n}}\sigma_{\hat{n}}^s(\bm x) \leq \varepsilon$ $\forall \bm x \in S_{\hat{n}}$. Then it follows that $S_{\hat{n}} \supseteq S_\varepsilon(\bm x_0)$ with probability at least $1 - \delta$.
\end{lemma}

\begin{proof}
We prove the claim by showing that $S_{\hat{n}} \supseteq R_\varepsilon^j(\bm x_0)$ $\forall j$, which we can show by induction. For the base case, let's take $j = 1$ and show that $S_{\hat{n}} \supseteq R_\varepsilon(\bm x_0)$. This inclusion follows directly from Lemma \ref{lemma:inclusion_for_per_gp_operator}, once we recall that the GP that generated $S_{\hat{n}}$, denoted as GP$_{\hat{\mathcal{D}}}$, is in $\beta\text{-GP}_f^\varepsilon(\bm x_0)$ with probability of at least $1 - \delta$ (thanks to Lemma \ref{lemma:in_beta_gp_with_high_prob}) and after noting that $S_{\hat{n}}$ is nothing else than $R^{\text{GP}_{\hat{\mathcal{D}}}}_\varepsilon(\bm x_0)$. For the induction step, let's assume that $S_{\hat{n}} \supseteq R_\varepsilon^{j-1}(\bm x_0)$ and show that $S_{\hat{n}} \supseteq R_\varepsilon^{j}(\bm x_0)$. The reasoning is completely analogous to the base case with the substitution $\bm x_0 \rightarrow R_\varepsilon^{j-1}(\bm x_0)$.
\end{proof}

Lemma \ref{lemma:safe_set_at_epsilon_includes_reachable_set} tells us that if there exists an iteration $\hat{n}$ at which $\beta_{\hat{n}}\sigma_{\hat{n}}$ is less or equal to $\varepsilon$ over $S_{\hat{n}}$, then with high probability we have classified as safe a region that includes the goal set $S_\varepsilon(\bm x_0)$. 

Lastly, before we can prove \cref{thm:exploration_convergence,thm:combined_convergence}, we need two more results about $\bar{S}(s)$, which we recall we defined as the total true safe set: $\bar{S}(s) \coloneqq \left\{\bm x \in \mathcal{X} : s(\bm x) \geq 0\right\}$.

\begin{lemma}\label{lemma:safe_set_in_s0}
For any GP$_\mathcal{D}$ such that $h(\bm x) \in [\mu_n(\bm x) \pm \beta_n\sigma_n(\bm x)]$ for all $\bm x$ and for all $n\leq |\mathcal{D}|$, it holds that $S_n \subseteq \bar{S}(s)$.
\end{lemma}
\begin{proof}
The claim is a direct consequence of the well behaved nature of the GP, as $h(\bm x) \in [\mu_n(\bm x) \pm \beta_n\sigma_n(\bm x)]$ implies a similar condition for the safety constraint $s$, so that if the lower bound of the confidence interval is above zero, so is the true value of the constraint $s$.
\end{proof}

\begin{lemma}\label{lemma:s_bar_in_s0}
$\forall \varepsilon > 0$ it holds that $S_\varepsilon(\bm x_0) \subseteq \bar{S}(s)$.
\end{lemma}
\begin{proof}
This result follows directly from Lemma \ref{lemma:safe_set_in_s0}, since all GPs in $\bm \beta$-GP$_s^\varepsilon\left(R_\varepsilon^j(\bm x_0)\right)$ satisfy by definition the assumption of \cref{lemma:safe_set_in_s0} $\forall j$.
\end{proof}

With these results, we can now show that, by starting with $S_n = \{\bm x_0\}$ and then sampling in $S_n$ according to an acquisition function that satisfies the assumptions of \cref{lemma:generic_form_vanishing_variance}, we will eventually classify the whole $S_\varepsilon(\bm x_0)$ as safe.

\begin{lemma}\label{lemma:convergence_bound_safe_set_generic}
Let the domain $\mathcal{X}$ be discrete, and assume that $\bar{S}(s)$ contains a finite number of elements: $|\bar{S}(s)| \leq D_{\max} < \infty$. Then if we choose the parameters to evaluate as the maximizers of an acquisition function that satisfies the assumptions of \cref{lemma:generic_form_vanishing_variance} it holds true that $\forall \varepsilon > 0$ there exist $N_{\varepsilon}$ and $n^* \leq D_{\max} N_{\varepsilon}$ such that, with probability of at least $1 - \delta$, $S_n \supseteq S_\varepsilon(\bm x_0)$ $\forall n \geq n^*$. The smallest $N_\varepsilon$ is given by:
\begin{equation}\label{eq:generic_smallest_n_epsilon}
N_\varepsilon = \min\left\{N \in \mathbb{N} : \beta_n g^{-1}\left(\frac{C_1\gamma_N}{N}\right) \leq \varepsilon\right\}
\end{equation}
where $\gamma_n$, the constant $C_1$ and the function $g$ are as in \cref{lemma:generic_form_vanishing_variance}.
\end{lemma}
\begin{proof}
Thanks to \cref{lemma:generic_form_vanishing_variance} we know that that $\forall \tilde{\varepsilon} > 0$, if the safe set stops expanding at iteration $\hat{n}$ it is possible to find $N_{\tilde{\varepsilon}}$ such that $\sigma_n \leq \tilde{\varepsilon}$ over the safe set $\forall n \geq \hat{n}+ N_{\tilde{\varepsilon}}$. Let $N_{\varepsilon/\beta}$ be such constant for the threshold $\varepsilon / \beta$, as, for example, the one in \cref{eq:generic_smallest_n_epsilon}, where, for simplicity, we have dropped the subscript for $\beta$, which should be $N_\varepsilon$. Because of the definition of $S_n$ as given in \cref{eq:safe_set_definition}, we have that $S_{n+1} \supseteq S_n$ $\forall n$, which means that $\forall n$ either $S_{n + N_{\varepsilon/\beta}} \supset S_n$ or $S_{n + N_{\varepsilon/\beta}} \setminus S_n = \emptyset$ and, therefore, $\beta\sigma_{n + N_{\varepsilon}/\beta} \leq \varepsilon$ over the safe set. In the second case, thanks to Lemma \ref{lemma:safe_set_at_epsilon_includes_reachable_set} we would have that, with probability of at least $1 - \delta$, $S_{n + N_{\varepsilon/\beta}} \supseteq S_\varepsilon(\bm x_0)$. On the other hand, in the first case, the worst case scenario is the one in which the safe set expands of a single element (i.e. $|S_{n+N_{\varepsilon/\beta}} \setminus S_n| = 1$). In that case, starting from $\bm x_0$, thanks to Lemma \ref{lemma:in_beta_gp_with_high_prob} and Lemma \ref{lemma:safe_set_in_s0} we know that if we add $D_{\max}$ elements to $S_n$, then, with probability of at least $1 - \delta$, $S_n = \bar{S}(s)$, which, thanks to Lemma \ref{lemma:s_bar_in_s0} implies that $S_n \supseteq S_\varepsilon(\bm x_0)$. The proof is then complete once we recall once more that $S_{n+1} \supseteq S_n$ $\forall n$.
\end{proof}

With \cref{lemma:convergence_bound_safe_set_generic} at hand, in order to prove \cref{thm:exploration_convergence} we now only need to show that the exploration component of our acquisition function, $\alpha^{\ourmethodexp}$, satisfies the assumptions of \cref{lemma:generic_form_vanishing_variance}.

To simplify the notation, in the following we define $\Delta \hat{H}_n(\bm x, \bm z) = \hat{I}_n\left(\{\bm x, y\}; \cpsi(\bm z)\right)$.

\begin{lemma}\label{lemma:monotonicity_ratio}
The mutual information $\hat{I}_n(\{\bm x, y\}; \cpsi(\bm z))$ as given by \cref{eq:approx_exploration_mutual_info} is monotonically decreasing in $\mu_n^2(\bm z) / \sigma_n^2(\bm z)$ $\forall \bm x, \bm z \in \mathcal{X}$.
\end{lemma}

\begin{proof}
First of all, let us simplify notation and define $R^2 \coloneqq \mu^2_n(\bm z) / \sigma^2_n(\bm z)$ and $\tilde{\rho}^2 \coloneqq \rho_\nu^2(\bm x)\rho^2_n(\bm x, \bm z)$. We then need to show that:
\begin{equation}\label{eq:partial_ratio}
\frac{\partial}{\partial R^2}\left[\exp\left\{-c_1R^2\right\} - \sqrt{\frac{1 - \tilde{\rho}^2}{1 + c_2\tilde{\rho^2}}}\exp\left\{-c_1R^2\frac{1}{1 + c_2\tilde{\rho}^2}\right\}\right] < 0 ~\forall R, ~\forall \tilde{\rho} \in [0, 1]
\end{equation}
We then can compute the derivative and ask under which conditions it is non negative. Requiring \cref{eq:partial_ratio} to be non negative is equivalent to ask that:
\begin{equation}
R^2 \leq \frac{1 + c_2\tilde{\rho}^2}{c_1c_2\tilde{\rho}^2}\left[\ln\left(1 + c_2\tilde{\rho}^2\right) + \frac{1}{2}\ln\left(\frac{1 + c_2\tilde{\rho}^2}{1 - \tilde{\rho}^2}\right)\right]
\end{equation}
Now, we observe that, since $c_2 \in (-1, 0)$, while $c_1>0$ and $\tilde{\rho}^2 \in [0, 1]$, the factor $(1 + c_2\tilde{\rho}^2) / c_1c_2\tilde{\rho}^2$ is always negative. For what concerns the sum of logarithms in the square brackets, it is strictly positive $\forall \tilde{\rho}^2 \in [0, 1]$, which means that, for \cref{eq:partial_ratio} to be non negative, we would need $R^2 < 0$, which is impossible, given that $R \in \mathbb{R}$.
\end{proof}

\begin{lemma}\label{lemma:delta_h_less_than_sigma}
$\forall n$, $\forall \bm x \in S_n$, $\forall \bm z \in \mathcal{X}$, it holds that
\begin{equation}
\Delta \hat{H}_n(\bm x, \bm z) \leq \ln(2) \frac{\sigma_n^2(\bm x)}{\sigma_\nu^2}.
\end{equation}
\end{lemma}

\begin{proof}
This result can be shown directly with the following inequality chain:
\begin{equation}
\begin{split}
\hat{I}_n\left(\{\bm x, y\}; \Psi(\bm z)\right) &\leq \ln(2)\left[1 - \sqrt{\frac{1 - \rho_\nu^2(\bm x)\rho_n^2(\bm x, \bm z)}{1 + c_2\rho_\nu^2(\bm x)\rho_n^2(\bm x, \bm z)}}\right] \\
&\stackrel{c_2 \in (-1, 0)}{\leq} \ln(2)\left(1 - \sqrt{1 - \rho_\nu^2(\bm x)\rho_{n}^2(\bm x, \bm z)}\right) \\
&\stackrel{\rho_n\rho_n^2 \in [0, 1]}{\leq} \ln(2)\left(\rho_\nu^2(\bm x)\rho_{n}^2(\bm x, \bm z)\right)\\\
&\stackrel{\rho_n^2 \in [0, 1]}{\leq} \ln(2)\rho_\nu^2(\bm x) \\
&\leq \ln(2)\frac{\sigma_n^2(\bm x)}{\sigma_\nu^2}
\end{split}
\end{equation}
where the first inequality follows from \cref{lemma:monotonicity_ratio}.
\end{proof}

With \cref{lemma:delta_h_less_than_sigma} we can now also prove \cref{lemma:mutual_infos_decrease_with_sigma}.

\MIsDecreaseWithSigma*
\begin{proof}
We prove this lemma by showing that $\alpha^{\ourmethodexp}(\bm x) \leq \ln(2)\frac{\sigma_{s,n}^2(\bm x)}{\sigma_{s,\nu}^2}$ and that $\alpha^{\mes}(\bm x) \leq \frac{1}{2}\ln\left(1 + \sigma_{f,\nu}^{-2}\sigma_{f,n}(\bm x)\right)$. The first inequality is nothing else that the result of \cref{lemma:delta_h_less_than_sigma}. For the second inequality, we fist show that
\begin{equation}\label{eq:mi_max_value_bounded}
I\left(f^*; \bm y_{D_n}\right) \leq I\left(f; \bm y_{D_n}\right),
\end{equation}
where $\bm y_{D_n}$ are the observed values up to iteration $n$: $[\bm y_{D_n}]_i = y_i$. To prove \cref{eq:mi_max_value_bounded} we recall that, for any three random variables $A, B, C$, it holds that $I\left(A; (B, C)\right) = I\left(A; C\right) + I\left(A; B | C\right)$. From this identity, it follows that we can rewrite $I\left(\bm y_D; (f, f^*)\right)$ in two equivalent ways:
\begin{equation}\label{eq:equivalence_data_proc_inequality}
I\left(f^*; \bm y_D\right) + I\left(\bm y_D; f| f^*\right) = I\left(\bm y_D; (f, f^*)\right) = I\left(\bm y_D; f\right) + I\left(\bm y_D; f^* | f\right).
\end{equation}
Now, if we recall the definition of $f^*$, it is immediate to see that $I(\bm y_D; f^* | f) = 0$, since $f^*$ is completely determined by $f$. The result in \cref{eq:mi_max_value_bounded} then follows from the fact that the mutual information is always non negative.
If we now recall that $I_n\left(f; \{\bm x, y\}\right) = \frac{1}{2}\ln\left(1 + \sigma_{f,\nu}^{-2}\sigma_{f,n}^2(\bm x)\right)$, we can use \cref{eq:mi_max_value_bounded} to derive that $I_n\left(f^*; \{\bm x, y\}\right) \leq \frac{1}{2}\ln\left(1 + \sigma_{f,\nu}^{-2}\sigma_{f,n}(\bm x)\right)$, which concludes the proof once we recall that $\ln(1 + x) \leq x$ $\forall x > 0$ and that $\ln(2) < 1$. 
\end{proof}

\begin{lemma}\label{lemma:mean_is_bounded}
Let $f$ be a real valued function on $\mathcal{X}$ and let $\mu_n$ and $\sigma_n$ be the posterior mean and standard deviation of a GP($\mu_0, k$) such that it exists a non-decreasing sequence of positive numbers $\{\beta_n\}$ for which $P\left\{f(\bm x) \in [\mu_n(\bm x) \pm \beta_n\sigma_n(\bm x)]~\forall \bm x,~\forall n\right\} \geq 1 - \delta$. Moreover, assume that $\mu_0(\bm x) = 0$ for all $\bm x$ and that $k(\bm x, \bm x^\prime) \leq 1$ for all $\bm x, \bm x^\prime \in \mathcal{X}$. Then it follows that $|\mu_n(\bm x)| \leq 2\beta_n$ with probability of at least $1 - \delta$ jointly for all $\bm x$ and for all $n$.
\end{lemma}

\begin{proof}
From the hypothesis, it follows that the following two conditions hold for all $\bm x$ and for all $n$ with probability of at least $1 - \delta$:
\begin{equation}
|f(\bm x)| \in \left[0, \beta_0\sigma_0(\bm x)\right]
\end{equation}
\begin{equation}
\mu_n(\bm x) \in \left[-|f(\bm x)| - \beta_n\sigma_n(\bm x), |f(\bm x)| + \beta_n\sigma(\bm x) \right]
\end{equation}
From these two conditions, it follows that $|\mu_n(\bm x)| \leq \beta_0\sigma_0(\bm x) + \beta_n\sigma_n(\bm x)$ with probability of at least $1 - \delta$. Now, we recall that the sequence $\{\beta_n\}$ is non decreasing by assumption and that the sequence $\{\sigma_n(\bm x)\}$ is non increasing by the properties of a GP, which allows us to conclude that $|\mu_n(\bm x)| \leq 2\beta_n\sigma_0(\bm x)$, which concludes the proof once we recall the assumption that $k(\bm x, \bm x^\prime) \leq 1$ for all $\bm x, \bm x^\prime \in \mathcal{X}$. The result can easily be extended to the case of non zero prior mean, by just adding the prior mean as offset in the found upper bound for the posterior mean.
\end{proof}

\begin{lemma}\label{lemma:delta_h_is_bounded_below}
$\forall n$, let $\tilde{\bm x} \in \argmax_{S_n}\sigma_n^2(\bm x)$ and define $\tilde{\sigma}^2 \coloneqq \sigma_n^2(\tilde{\bm x})$, then it holds that:
\begin{equation}\label{eq:delta_h_lower_bound}
\Delta \hat{H}_n(\bm x_{n+1}, \bm z_{n+1}) \geq \eta(\tilde{\sigma}^2),
\end{equation}
with probability of at least $1 - \delta$ for all $\bm x, \bm z$ and for all $n$, and with $\eta$ given by
\begin{equation}\label{eq:eta_definition}
\eta(x) \coloneqq \ln(2)\exp\left\{-c_1\frac{4\beta^2}{x}\right\}\left[1 - \sqrt{\frac{\sigma_\nu^2}{2c_1x + \sigma_\nu^2}}\right].
\end{equation}
\end{lemma}

\begin{proof}
First we note tat this Lemma is only non-trivial in case the posterior mean is bounded on $S_n$, otherwise, if we admit $|\mu_n(\bm x)| \to \infty$, then we just recover the result that the average information gain is positive.

Now, moving to the proof, as first thing we recall that the exploration component of our algorithm always selects the $\argmax_{\bm x \in S_n}$ of $\hat{I}_n\left(\{\bm x, y\}; \Psi(\bm z)\right)$ as next  parameter to evaluate, meaning that, by construction, $\forall \bm x \in S_n$ and $\forall \bm z \in \mathcal{X}$, it holds that:
\begin{equation}
\Delta \hat{H}_n(\bm x_{n+1}, \bm z_{n+1}) \geq \Delta \hat{H}_n(\bm x, \bm z)
\end{equation}
This implies, in particular, that $\Delta \hat{H}_n(\bm x_{n+1}, \bm z_{n+1}) \geq \Delta \hat{H}_n(\tilde{\bm x}, \tilde{\bm x})$, since, by definition, $\tilde{\bm x} \in S_n$ and is, therefore, always feasible. Writing this condition explicitly, we obtain:
\begin{equation}
\begin{split}
\Delta \hat{H}_n(\bm x_{n+1}, \bm z_{n+1}) &\geq \Delta \hat{H}_n(\tilde{\bm x}, \tilde{\bm x}) \\
&= \ln(2)\left[\exp\left\{-c_1\frac{\mu_n^2(\tilde{\bm x})}{\tilde{\sigma}^2}\right\} - \sqrt{\frac{1 - \rho_\nu^2(\tilde{\bm x})}{1 + c_2\rho_\nu^2(\tilde{\bm x})}}\exp\left\{-c_1\frac{\mu_n^2(\tilde{\bm x})}{\tilde{\sigma}^2}\frac{1}{1 + c_2\rho_\nu^2(\tilde{\bm x})}\right\}\right] \\
&\geq \ln(2)\exp\left\{-c_1\frac{4\beta^2}{\tilde{\sigma}^2}\right\}\left[1 - \sqrt{\frac{1 - \rho_\nu^2(\tilde{\bm x})}{1 + c_2\rho_\nu^2(\tilde{\bm x})}}\right] \\
&= \eta(\tilde{\sigma}^2)
\end{split}
\end{equation}
where we have used the fact that $c_2 \in (-1, 0)$ and that $\rho_\nu^2(\tilde{\bm x}) \in [0, 1]$, in addition to \cref{lemma:mean_is_bounded} for the last inequality.
\end{proof}

\cref{lemma:delta_h_less_than_sigma,lemma:delta_h_is_bounded_below} show that $\alpha^{\ourmethodexp}$ satisfies the assumptions of \cref{lemma:generic_form_vanishing_variance}, so that we are now able to prove \cref{thm:exploration_convergence}.

\ThmExplorationConvergence*
\begin{proof}
From \cref{lemma:delta_h_less_than_sigma,lemma:delta_h_is_bounded_below} it follows that $\alpha^{\ourmethodexp}$ satisfies the conditions of \cref{lemma:generic_form_vanishing_variance} with $C_0 = \ln(2) / \sigma^2_\nu$ and with $g = b$, as given by \cref{eq:eta_definition}. This fact in turn implies that we can also apply \cref{lemma:convergence_bound_safe_set_generic} to it, which concludes the proof, ones we substitute the correct values to $C_0$ and $g$.
\end{proof}

\LemmaMesEquivalentAcquisition*
\begin{proof}
Approximating \cref{eq:mes_acquisition_expression} as a Monte Carlo average yields
\begin{equation}\label{eq:mes_as_montecarlo}
I_n\left(\{\bm x, y\}; f^*\right) \approx \sum_{i = 1}^m\left[\frac{\theta_{y^*_i}(\bm x)\psi(\theta_{y_i^*}(\bm x))}{2\Psi(\theta_{y_i^*}(\bm x))} - \ln\left(\Psi(\theta_{y_i^*}(\bm x))\right)\right],
\end{equation}
where $\left\{y_i^*\right\}_{i=i}^m$ are $m$ samples of $f_{S_n}^*$. If now we restrict the sum to a single sample, we get:
\begin{equation}\label{eq:mes_as_montecarlo_one_sample}
I_n\left(\{\bm x, y\}; f^*\right) \approx \frac{\theta_{y^*}(\bm x)\psi(\theta_{y^*}(\bm x))}{2\Psi(\theta_{y^*}(\bm x))} - \ln\left(\Psi(\theta_{y^*}(\bm x))\right).
\end{equation}
As \citet{wang_max-value_2017} show in their Lemma 3.1, maximizing \cref{eq:mes_as_montecarlo_one_sample} is equivalent to minimizing $\theta_{y^*}$, which, in turn, is equivalent to maximizing $1/\theta_{y^*}^2$. The proof is then complete once we recall that $\theta_{y^*}(x) = \sigma_n(\bm x)^{-1} \left(y^* - \mu_n(\bm x)\right)$.
\end{proof}

\ThmCombinedConvergence*
\begin{proof}
In order to prove the theorem, we just need to show that the aquisition function $\alpha_{\text{combined}} \coloneqq \max \left\{\alpha^{\ourmethodexp}(\bm x), \hat{\alpha}_\phi^{\mes}(\bm x)\right\}$ satisfies the assumptions of \cref{lemma:generic_form_vanishing_variance}. We can then apply a reasoning analogous to the proof of \cref{lemma:convergence_bound_safe_set_generic} to show that at most after $N_S N_\varepsilon$ iterations $S_n \supseteq S_\varepsilon(\bm x_0)$ and, therefore, it contains $\bm x_\varepsilon^*$. The second part of the theorem follows by recalling the result of \cref{lemma:safe_set_in_s0} and by noticing that the worse case scenario happens when a single new parameter is added to the safe set after exactly $N_\varepsilon$ iterations and the last parameter to be added is precisely $\bm x_\varepsilon^*$. In that case after adding $\bm x_\varepsilon^*$ to $S_n$, the safe set cannot expand any more, so that after further $N_\varepsilon$ iterations, the posterior variance will be smaller than $\varepsilon$ over $S_n$ and, therefore, also at $\bm x_\varepsilon^*$
Using \cref{lemma:mean_is_bounded}, it is straightforward to see that $\hat{\alpha}_\phi^{\mes}$ with high probability satisfies the assumptions of \cref{lemma:convergence_bound_safe_set_generic}, since $\hat{\alpha}_\phi^{\mes} \leq \sigma_n^2(\bm x)(\phi - 2\beta)^{-2}$, and $\hat{\alpha}_\phi^{\mes}(\bm x_{n+1}) \geq \sigma_n(\tilde{\bm x}) / \phi$, were $\tilde{\bm x} = \argmax_{\bm x \in S_n}\sigma_n(\bm x)$.

Combining these inequalities with the analogous results for $\alpha^{\ourmethodexp}$ from \cref{lemma:delta_h_less_than_sigma,lemma:delta_h_is_bounded_below}, we can conclude that the combined acquisition function satisfies the assumptions of \cref{lemma:generic_form_vanishing_variance} with
\begin{itemize}
\item $C_0 = \max\left\{\frac{\ln(2)}{\sigma_\nu^2}, \frac{1}{\phi - 2\beta}\right\}$,
\item $g(\tilde{\sigma}_n^2) = \min\left\{\eta(\tilde{\sigma}_n^2), \frac{\tilde{\sigma}_n^2}{\phi}\right\}$,
\end{itemize}
where $\eta$ is the function defined in \cref{eq:eta_definition}. The proof is then complete once we recall that the $\min$ of two monotonically increasing functions is also monotonically increasing.
\end{proof}

As already noticed in \cref{sec:ise_theory}, the convergence results of both \cref{thm:exploration_convergence,thm:combined_convergence} depends on the existence of $N_\varepsilon$ as given in \cref{eq:N_epsilon}. For constant $\beta$, the existence result is trivial, given the monotonicity of the involved functions. On the other hand, when $\beta$ grows with $n$, this is no longer obvious. However, for discrete domains and for the choice of $\beta_n$ given by \cref{eq:beta_def}, it is also possible for $N_\varepsilon$ to exist.
\cref{fig:eta_convergence} shows an example of the asymptotic behavior of the logarithm of $\eta\left(\varepsilon / \beta_n\right) n / C \gamma_n$. We see that, for big enough $n$, this logarithm exceeds the threshold of zero, marked by the orange dashed line, which is equivalent to the condition in \cref{eq:N_epsilon} being satisfied. 

\begin{figure}[t] 
  \centering
  \includegraphics[width=0.47\textwidth]{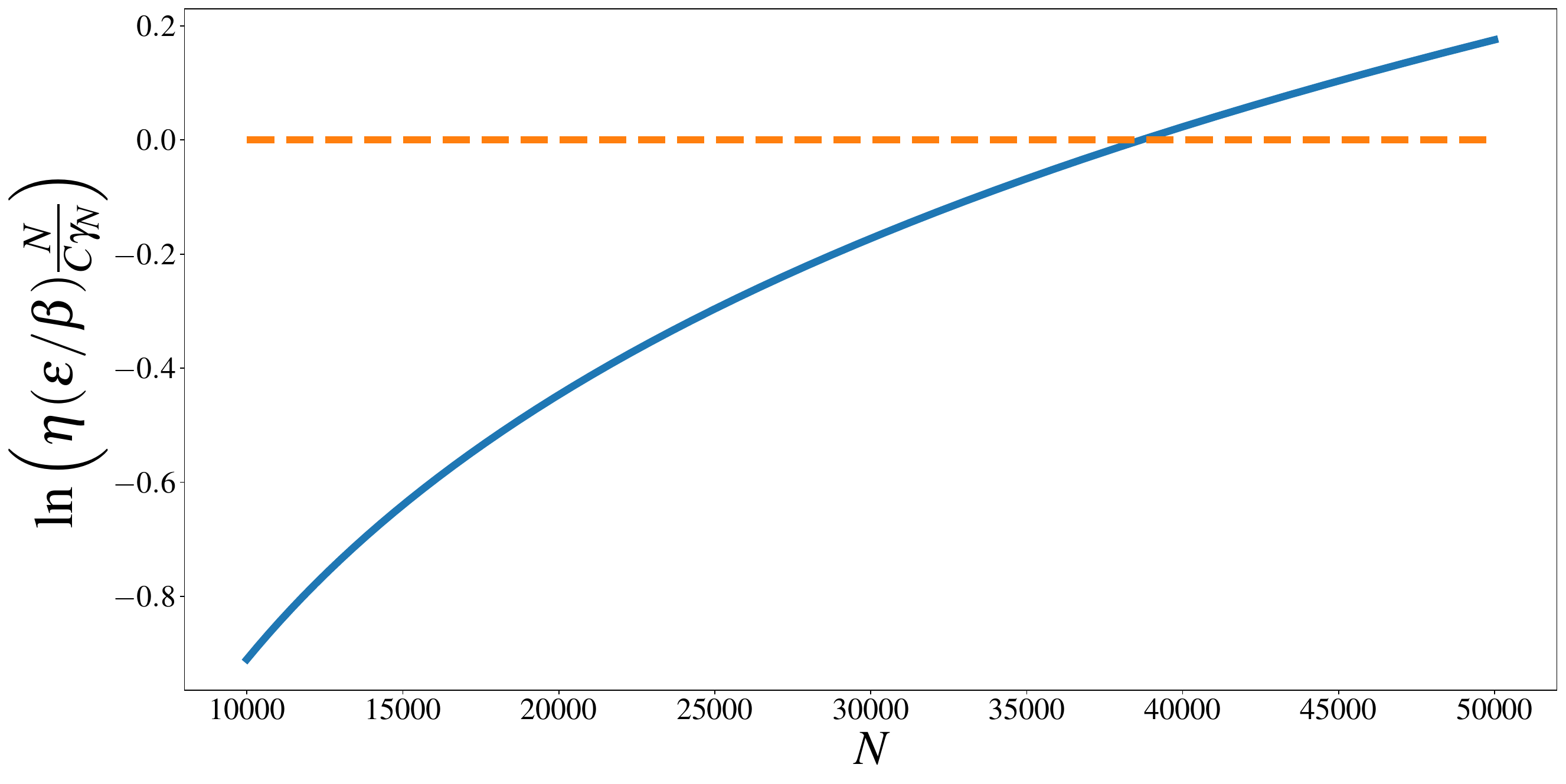}
  \caption{Asymptotic behavior of the logarithm of $\eta\left(\varepsilon / \beta_n\right) n / C \gamma_n$ (blue line), compared to $\ln(1)$ (orange dashed line). We see that for big enough $n$ the blue line crosses the orange one, which is equivalent to the condition in \cref{eq:N_epsilon} being satisfied.}
  \label{fig:eta_convergence}
\end{figure}

\section{Entropy of \texorpdfstring{$\cpsi(\bm{z})$}{\b} approximation}\label{appendix:entropy_approx}

In order to analytically compute the mutual information $\hat{I}_n\left(\{\bm x, y\}; \cpsi(\bm z)\right) = H_n\left[\cpsi(\bm z)\right] - \mathbb{E}_{y}\left[H_{n+1}\left[\cpsi(\bm z) \middle| \{\bm x, y\}\right]\right]$, we have approximated the entropy of the variable $\cpsi(\bm z)$ at iteration $n$ with $\hat{H}_n\left[\cpsi(\bm z)\right]$, given by \cref{eq:approximated_psi_entropy}, which we have then used to derive the results presented in the paper. The approximation allowed us to derive a closed expression for the average of the entropy at parameter $\bm z$ after an evaluation at $\bm x$, $\mathbb{E}_{y}\left[\hat{H}_{n+1}\left[\cpsi(\bm z) \middle| \{\bm x, y\}\right]\right]$. 
This approximation was obtained by noticing that the exact entropy \cref{eq:exact_psi_entropy} has a zero mean Gaussian shape, when plotted as function of $\mu_n(\bm x) / \sigma_n(\bm x)$, and then by expanding both the exact expression \cref{eq:exact_psi_entropy} and a  zero mean un-normalized Gaussian in $\mu_n(\bm x) / \sigma_n(\bm x)$ in their Taylor series around zero. At the second order we obtain, respectively,
\begin{equation}\label{eq:entropy_first_order}
H_n\left[\cpsi(\bm x)\right] = \ln(2) - \frac{1}{\pi}\left(\frac{\mu_n(\bm x)}{\sigma_n(\bm x)}\right)^2 + o\left(\left(\frac{\mu_n(\bm{x})}{\sigma_n(\bm{x})}\right)^2\right)
\end{equation}
and 
\begin{equation}\label{eq:gaussian_first_order}
c_0 \exp\left\{-\frac{1}{2\sigma^2}\left(\frac{\mu_n(\bm{x})}{\sigma_n(\bm{x})}\right)^2\right\} = c_0 - c_0\frac{1}{2\sigma^2}\left(\frac{\mu_n(\bm{x})}{\sigma_n(\bm{x})}\right)^2 + o\left(\left(\frac{\mu_n(\bm{x})}{\sigma_n(\bm{x})}\right)^2\right).
\end{equation}
By equating the terms in \cref{eq:entropy_first_order} with the ones in \cref{eq:gaussian_first_order}, we find ${c_0 = \ln(2)}$ and ${\sigma^2 = \ln(2)\pi/2}$, which leads to the approximation for $H_n\left[\cpsi(\bm z)\right]$ \cref{eq:approximated_psi_entropy} used in the paper: $H_n\left[\cpsi(\bm z)\right] \approx \ln(2) \exp\left\{-\frac{1}{\pi\ln(2)}\left(\frac{\mu_n(\bm z)}{\sigma_n(\bm z)}\right)^2\right\}$.
In \cref{fig:psi_entropy_comparison} we plot \cref{eq:exact_psi_entropy} and \cref{eq:approximated_psi_entropy} against each other as a function of the mean-standard deviation ratio, while \cref{fig:psi_entropies_diff} shows the difference between the two. From these two plots, one can see almost perfect agreement between the two functions, with a non negligible difference limited to two small neighborhoods of the $\mu/\sigma$ space. 

\begin{figure}[ht]
\hfill
     \centering
     \begin{subfigure}[b]{0.49\textwidth}
         \centering
         \includegraphics[height=0.15\textheight]{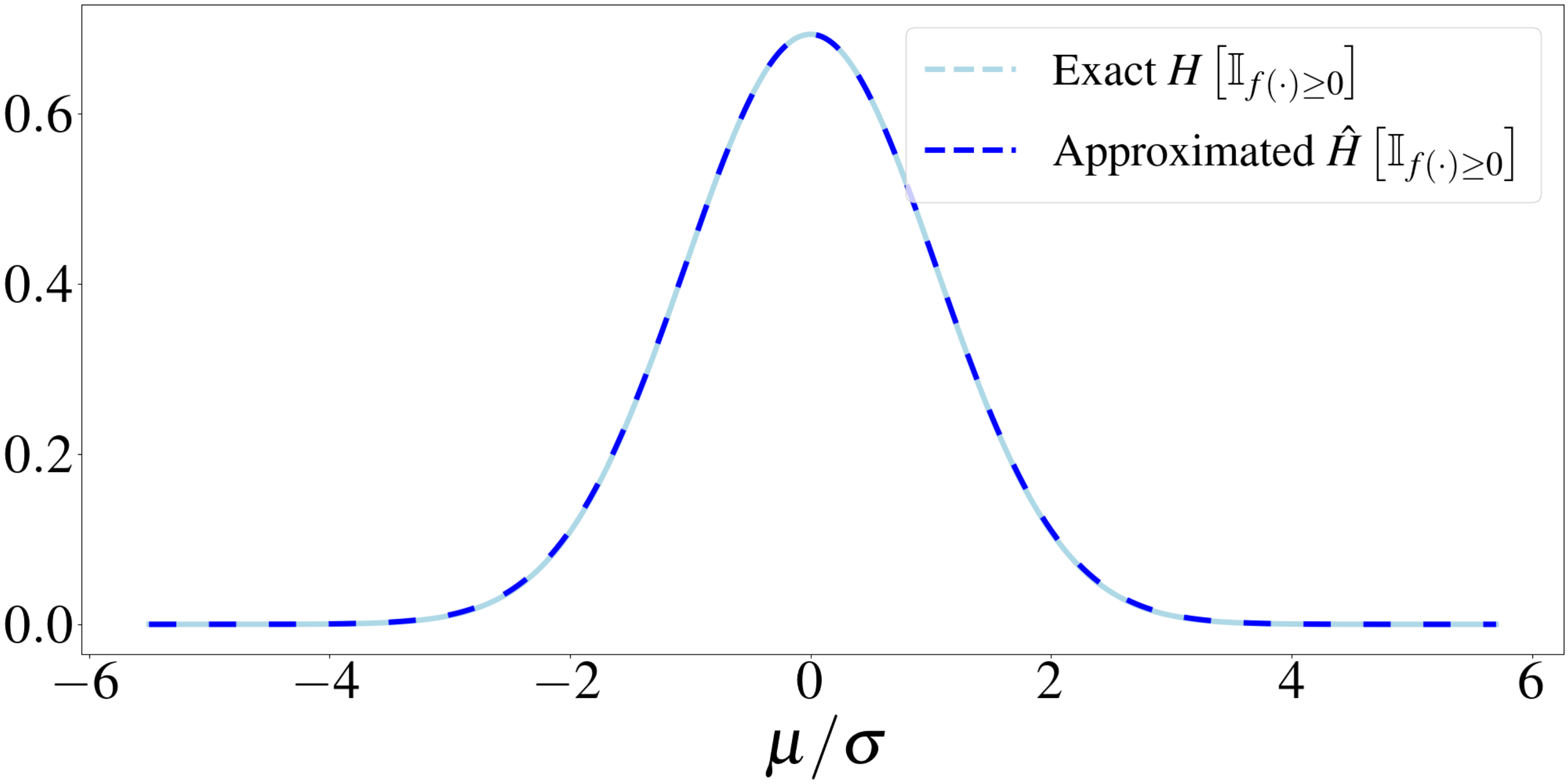}
         \caption{Exact and approximated $H[\cpsi(\bm z)]$.}
         \label{fig:psi_entropy_comparison}
     \end{subfigure}
     \hfill
     \begin{subfigure}[b]{0.49\textwidth}
         \centering
         \includegraphics[height=0.15\textheight]{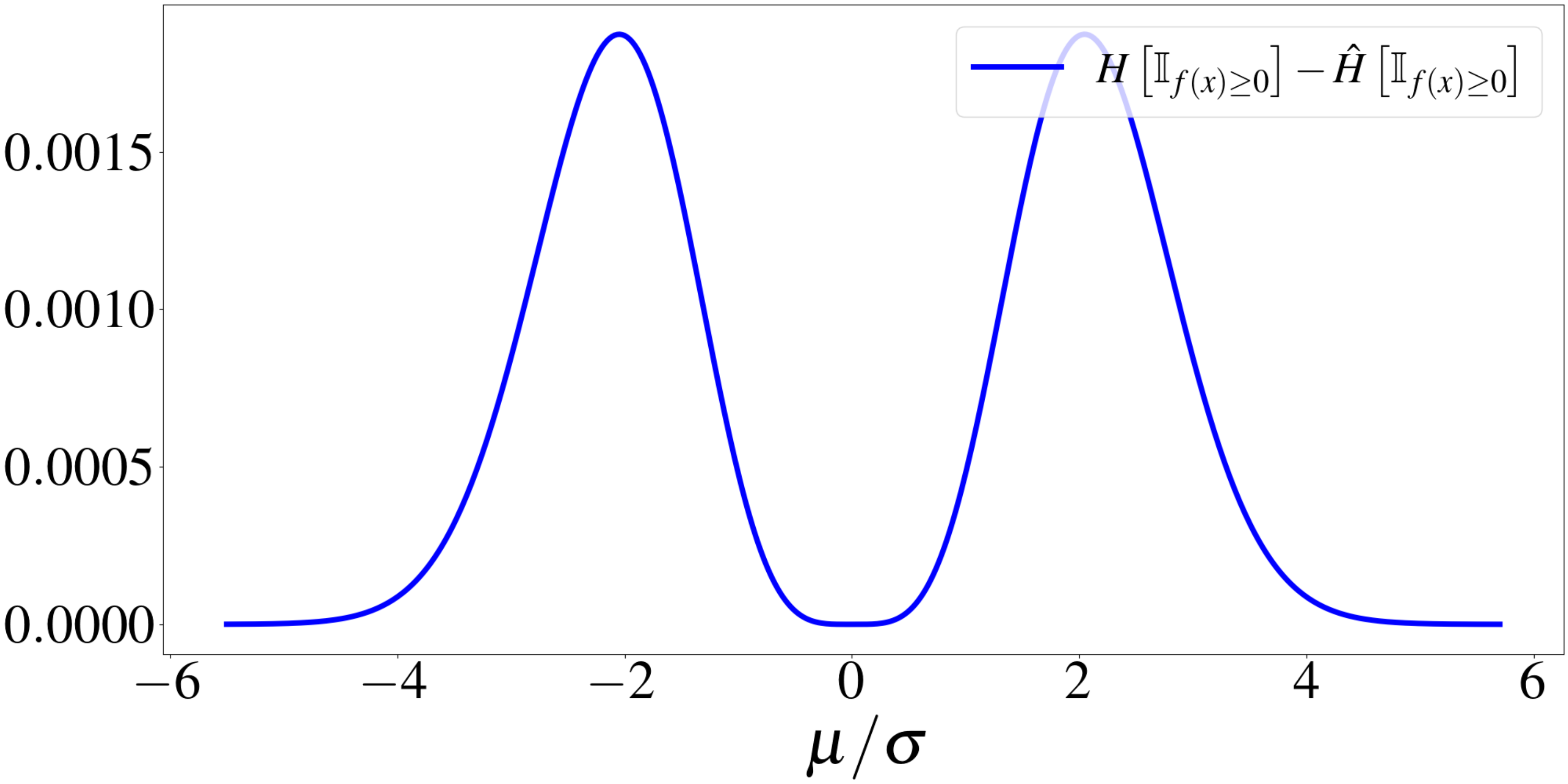}
         \caption{Approximation error.}
         \label{fig:psi_entropies_diff}
     \end{subfigure}
        \caption{Comparison between exact entropy of $\cpsi(\bm x)$ \cref{eq:exact_psi_entropy} and approximated form \cref{eq:approximated_psi_entropy}: (a) the two entropies plotted against each other; (b) the approximation error expressed as difference of the two.}
        \label{fig:psi_entropies}
        \hfill
\end{figure}

\section{Experiments Details}\label{appendix_experiments}

In this Appendix, we collect details about the experiment presented in \cref{sec:experiments}. Code for the used acquisition functions can be found at \url{https://github.com/boschresearch/information-theoretic-safe-exploration}. For the \mes component, we used the BoTorch \citep{botorch2020} implementation.

\paragraph{GP samples}
For the results shown in \cref{fig:gp_samples_exp} we run all methods on 50 samples from a GP defined on the square $[-1, 1] \times [-1, 1]$, with RBF kernel with the following hyperparameters: $\mu_0 \equiv 0$; kernel lengthscale = $0.3$; kernel outputscale = 30; $\sigma_\nu^2 = 0.05$, while the safe seed $\bm x_0$ was chosen as the origin: $\bm x_0 = (0, 0)$. As \safeopt requires a discretized domain, we used the same uniform discretization of 150 points per dimension for all GP samples.
Concerning the safety violations summarized in \cref{table:safety_violations}, the fact that they are comparable is expected, since in our experiments they all use the posterior GP confidence intervals to define the safe set. 

\paragraph{Synthetic one-dimensional objective}
The objective/constraint function we used in this experiment, which we plot in \cref{fig:one_d_exp_objective}, is given by $f(\bm x) = e^{-\bm x} + 15 e^{-(\bm x  - 4)^2} + 3e^{-(\bm x - 7)^2} + 18e^{-(\bm x - 10)^2} + 0.41$. The GP was defined on the domain $\mathcal{X} = [-2.4, 10.5]$, with RBF kernel with the following hyperparameters: $\mu_0 \equiv 0$; kernel lengthscale = $0.6$; kernel outputscale = $50$; $\sigma_\nu^2 = 0.05$, while the safe seed $\bm x_0$ was chosen as the origin: $\bm x_0 = (0, 0)$.

\paragraph{Heteroskedastic noise domains}
In these experiments we used the same \linebo wrapper as in the five dimensional experiment. The GP had a RBF kernel with hyperparameters: $\mu_0 \equiv 0$; kernel lengthscale = $1.6$; kernel outputscale = 1; while the safe seed $\bm x_0$ was set to the origin. For what concerns the observation noise, as explained in \cref{par:high_heteroskedastic}, we used heteroskedastic noise, with two different values of the noise variance in the two symmetric halves of the domain. In particular, given a parameter $\bm x = (x_1, x_2, \dots, x_n)$, we set the noise variance to $\sigma_\nu^2 = 0.05$ if $x_0 \geq 0$, otherwise we set it to $\sigma_\nu^2 = 0.5$. As explained in \cref{sec:experiments}, the constraint function is $f(\bm x) = \frac{1}{2}e^{-\bm x^2} + e^{-(\bm x \pm \bm x_1)^2} + 3e^{-(\bm x \pm \bm x_2)^2} + 0.2$, with $\bm x_1$ and $\bm x_2$ given by: $\bm x_1 = (2.7, 0, \dots, 0)$ and $\bm x_2 = (6, 0, \dots, 0)$.

\paragraph{OpenAI Gym Control}
For the OpenAI Gym pendulum experiments, we used the environment provided by the OpenAI gym \citep{brockman2016openai} under the MIT license. For the safety condition, the threshold angular velocity $\dot{\theta}_M$ was set to $0.5$ rad$/s$, with an episode length of 400 steps, and \cref{fig:pendulum_experiment} shows one run of \ourmethod using a GP with RBF kernel with the following hyperparameters: $\mu_0 \equiv 0$; kernel lengthscale = $1.3$; kernel outputscale = 6.6; $\sigma_\nu^2 = 0.04$.

\paragraph{Rotary inverted pendulum controller}
For the Furuta pendulum experiments, we used the simulator developed by \citet{quanser_sim} under the MIT license. For both the objective and the safety constraint, we used GPs prior with RBF kernel with the following hyperparameters: $\mu_0 \equiv 0$; kernel lengthscale = $0.3$; kernel outputscale = 50; $\sigma_\nu^2 = 0.05$.

\end{document}